\documentclass[hidelinks,twoside]{article}

\usepackage[accepted]{aistats2022}

\usepackage[round]{natbib}

\usepackage[utf8]{inputenc} %
\usepackage[T1]{fontenc}    %
\usepackage{hyperref}       %
\usepackage{url}            %
\usepackage{booktabs}       %
\usepackage{amsfonts}       %
\usepackage{nicefrac}       %
\usepackage{microtype}      %
\usepackage{xcolor}         %

\usepackage{breqn}

\usepackage{amsmath,amssymb,amsthm}
\usepackage{comment}
\usepackage{xcolor}
\usepackage{graphicx}
\usepackage{dsfont}
\usepackage{upgreek}
\usepackage{siunitx}

\usepackage[ruled,vlined]{algorithm2e}
\SetKwComment{Comment}{$\triangleright$\ }{}

\usepackage{caption}
\usepackage{subcaption}
\usepackage{float}
\usepackage{physics}
\usepackage{graphbox}
\usepackage{mathrsfs}

\newtheorem{remark}{Remark}

\usepackage{mathtools}

\DeclareMathOperator{\EX}{\mathbb{E}}%

\DeclareMathOperator{\lqr}{\mathcal{L}_\mathrm{QR}}%

\newcommand{\codecx}[1]{{\tt \text{#1}}\xspace} %
\newcommand{\dqnzoo}{\codecx{DQN\_ZOO}}

\newcommand{\simd}{\overset{D}{\equiv}}

\DeclareMathOperator*{\argmax}{arg\,max}

\newcommand{\ellpp}{\ell_p^p}
\newcommand{\ellp}{\ell_p}
\newcommand{\pip}{\Pi_{\ellp}}

\newcommand{\uij}{u_{ij}}
\newcommand{\uii}{u_{ii}}

\newcommand{\dij}{\delta_{ij}}
\newcommand{\dii}{\delta_{ii}}

\newcommand{\auij}{\abs{\uij}}
\newcommand{\auii}{\abs{\uii}}

\newcommand{\htaui}{\hat{\tau}_i}

\newcommand{\R}{\mathbb{R}}

\newcommand{\ind}{\mathds{1}}

\newcommand{\T}{T}

\newcommand{\TF}{\bar{F}}
\newcommand{\Ttheta}{\bar{\theta}}

\newcommand{\vtheta}{{\boldsymbol{\uptheta}}}
\newcommand{\vTtheta}{{\boldsymbol{\bar{\uptheta}}}}

\usepackage{thmtools}
\usepackage{thm-restate}

\begin{document}

\twocolumn[

\aistatstitle{A Cramér Distance perspective on Quantile Regression based Distributional Reinforcement Learning}

\aistatsauthor{ Alix Lhéritier \And Nicolas Bondoux }

\aistatsaddress{ Amadeus SAS, F-06902 Sophia Antipolis, France } ]

\begin{abstract}
Distributional reinforcement learning (DRL) extends the value-based approach by approximating the full distribution over future returns instead of the mean only, providing a richer signal that leads to improved performances. 
Quantile Regression (QR) based methods like QR-DQN project arbitrary distributions into a parametric subset of staircase distributions by minimizing the 1-Wasserstein distance. However,  due to biases in the gradients, the quantile regression loss is used instead for training, guaranteeing the same minimizer and enjoying unbiased gradients. 
Non-crossing constraints on the quantiles have been shown to improve the performance of QR-DQN for uncertainty-based exploration strategies. 
The contribution of this work is in the setting of fixed quantile levels and is twofold. First, we prove that the Cramér distance yields a projection that coincides with the 1-Wasserstein one and that, under non-crossing constraints, the squared Cramér and the quantile regression losses yield collinear gradients, shedding light on the connection between these important elements of DRL. Second, we propose a low complexity  algorithm to compute the Cramér distance.

\end{abstract}

\section{INTRODUCTION}

Distributional Reinforcement Learning (DRL) extends the value-based approach of DQN \citep{dqn} by considering the full distribution of returns as a learning signal allowing to take into account all the complexity of the randomness coming from the rewards, the transitions and the policy, which is hidden when considering the mean only. Even when a policy aims at maximizing the expected return, considering the full distribution provides an advantage in the presence of approximations, allowing to learn better representations and helping to reduce state aliasing \citep{c51}. With this new approach comes a generalization of the \emph{Bellman operator}---the \emph{distributional} Bellman operator---, whose contraction properties are key for guaranteeing the stability of DRL algorithms.

How distributions are represented and learned is also a key point, since some choices can break the contraction property \citep[Lemma 2]{rowland2018analysis}. Some approaches use staircase parametric representations whose steps correspond to fixed quantile values like in C51 \citep{c51} or to fixed quantile levels like in QR-DQN \citep{qrdqn}. Alternatively, FQN \citep{fqn} fully parameterize the staircase distributions.  IQN \citep{iqn} follows a different approach by approximating the quantile function with a neural network that takes the quantile level as input and must therefore be sampled during training.

DRL methods resort to different notions of distance or divergence between distributions in order to practically learn them but also to analyze the effect on the contraction property of the distributional Bellman operator.
In \cite{rowland2018analysis}, a Hilbert space endowed with the $\ell_2$ norm on cumulative distribution functions has been shown to be a natural framework to analyze the effect of the fixed quantile value representation of C51.
In \cite{cramerGAN},  the squared $\ell_2$ distance, called \emph{Cramér distance} in that work,\footnote{In this work, we follow \cite{rowland2018analysis} and use the term Cramér distance for the $\ell_2$ distance.} has been proposed  for Generative Adversarial Networks but also for machine learning in general due its unbiased gradients. 
In \cite{qrdqn}, the Wasserstein distance has been used for defining how a general distribution should be represented with fixed quantile levels and also to analyze the effect on the contraction property of the distributional Bellman operator. However, due to the biased gradients of the Wasserstein distance, the quantile regression loss is used to train the network, guaranteeing the same minimizer as the 1-Wasserstein distance and enjoying unbiased gradients.

When estimating multiple quantiles, one faces the issue of crossing quantiles, i.e., a violation of the monotonicity of the quantile function.
In QR-DQN, crossing quantiles make the learning signal noisy, affecting disambiguation of states as shown in \cite{nc-qrdqn}.
This issue has been addressed in the statistical literature of quantile regression (see, e.g. \cite{koenker1994quantile,he1997quantile,liu2009stepwise,hall1999methods,dette2008non,bondell2010noncrossing}) but also, more generally, in the machine learning literature on how to represent and learn monotonic functions (see, e.g., \cite[Table 1]{gupta2016monotonic}), with different approaches like including penalties in the loss function or enforcing monotonicity by design. 
Methods that take sampled quantile levels as input during training like \cite{sqr} or \cite{iqn}, have been shown to alleviate the problem.
In the DRL literature, \cite{nc-qrdqn,ndqfn} enforce monotonicity with special neural network designs obtaining improved results with respect to QR-DQN, in the setting of uncertainty-based exploration.

In this work, we analyze QR-based methods from a Cramér distance perspective and propose its square as an alternative loss function. 
In Section \ref{sec:background}, we expose the necessary background. 
In Section \ref{sec:equiv-wasserstein}, we show that the Cramér distance projection coincides with the 1-Wasserstein one, yielding a contraction guarantee.
In Section \ref{sec:equiv-qrloss}, we propose an alternative expression of the Cramér distance allowing to show that the QR and the Cramér losses are essentially equivalent for gradient based optimization under monotonicity constraints.
In Section \ref{sec:algorithm}, we propose another alternative expression of the Cramér distance based on quantile sorting, leading to an $O(N \log N)$ algorithm in contrast to the $O(N^2)$ complexity of the QR loss.
In Section \ref{sec:experiments}, we experimentally compare the different losses, illustrating the theory and the algorithm but also hinting at future research directions discussed in Section \ref{sec:conclusion}.

\section{BACKGROUND}
\label{sec:background}

We consider the classical model of agent-environment interactions \citep{puterman2014markov}, i.e.,  a Markov Decision Process (MDP)  $(\mathcal{S}, \mathcal{A}, R, P, \gamma)$, with $\mathcal{S}$ and $\mathcal{A}$ being the state and action space, $R: \mathcal{S} \times \mathcal{A} \rightarrow \mathbb{R}$ being the reward function, $P(s'|s,a): \mathcal{S} \times \mathcal{A} \times \mathcal{S} \rightarrow[0,1]$ being the probability of transitioning from state $s$ to state $s'$ after taking action $a$ and $\gamma \in[0,1)$ the discount factor. 
A stochastic policy $\pi(\cdot|s): \mathcal{S} \times \mathcal{A} \rightarrow[0,1]$  maps  a state $s$ to a distribution over $\mathcal{A}$. 

\subsection{Q-Learning}

For a fixed policy $\pi$, the \emph{return} $Z^{\pi}(s, a)$ is a random variable (RV) representing the discounted cumulative rewards the agent gains from a state $s$ by taking the action $a$  and then  following the policy $\pi$, i.e., $Z^{\pi}(s, a)\equiv\sum_{t=0}^{\infty} \gamma^{t} R\left(s_{t}, a_{t}\right)$ with $s_{0}=s, a_{0}=a$ and $s_{t+1} \sim P\left(\cdot \mid s_{t}, a_{t}\right), a_{t} \sim \pi\left(\cdot \mid s_{t}\right)$. 
The usual goal in reinforcement learning (RL) is to find an optimal policy $\pi^*$ maximizing the \emph{state-action value function} $Q^\pi(s,a)\equiv \EX{Z^{\pi}(s, a)}$, i.e., $Q^{\pi^*}(s,a)=\max_\pi Q^\pi(s,a)\equiv Q^*(s,a)\; \forall s,a$.
\emph{Q-Learning} \citep{watkins1992q} is an off-policy reinforcement learning algorithm
that directly learns the optimal state-action value
function  using the \emph{Bellman optimality operator}
\begin{equation}
(\mathcal{T} Q)(s, a) \equiv\mathbb{E} R(s, a)+\gamma \mathbb{E}_{P} \max _{a^{\prime} \in \mathcal{A}} Q\left(s^{\prime}, a^{\prime}\right).
\end{equation}
In the evaluation case, the \emph{Bellman operator} $\mathcal{T}^\pi$ \citep{Bellman:1957,watkins1992q}  is defined as
\begin{equation}(\mathcal{T}^{\pi} Q)(s, a) \equiv\mathbb{E} R(s, a)+\gamma \underset{P, \pi}{\mathbb{E}} Q\left(s^{\prime}, a^{\prime}\right)
.
\end{equation}
These operators are contractions and their repeated application to some initial value function $Q_0$
converges exponentially to  $Q^*$ or $Q^\pi$, respectively \citep{bertsekas1996neuro}.
However, when $Q$ is represented by a neural network  that is trained on batches of sampled transitions $(s,a,r,s')$ as in most deep learning studies, a gradient update is preferred since it allows for the dissipation of noise introduced in the target by stochastic approximation \citep{bertsekas1996neuro,kushner2003stochastic}.  
DQN \citep{dqn} iteratively trains the network by minimizing the squared \emph{temporal difference (TD)} error $\frac{1}{2}\left[r+\gamma \max _{a^{\prime}} Q_{\omega^{-}}\left(s^{\prime}, a^{\prime}\right)-Q_{\omega}(s, a)\right]^{2}$ 
over samples $(s,a,r,s')$, where $\omega^-$ is the target network, which is a copy of $\omega$, synchronized with it periodically.
When using an \emph{$\varepsilon$-greedy policy}, the samples are obtained while the agent interacts with the environment choosing actions uniformly at random with probability $\varepsilon$ and otherwise according to $\arg \max _{a} Q_\omega(s, a)$.

\subsection{Distributional reinforcement learning}

In order to extend the previous concepts to DRL, the {distributional Bellman operator} and \emph{optimality operator} \citep{c51} are defined as 
\begin{align}
&(\mathcal{T}^{\pi} Z)(s, a) \simd R(s, a)+\gamma Z\left(s^{\prime}, a^{\prime}\right),  a^{\prime} \sim \pi\left(\cdot \mid s^{\prime}\right)\label{eq:drl-bellman-op}\\
&(\mathcal{T} Z) (s, a)\simd R(s, a)+\gamma Z\left(s^{\prime}, \argmax_{a^{\prime} \in \mathcal{A}}  \mathbb{E}_p Z\left(s^{\prime}, a^{\prime}\right)\right)\nonumber\\
&\text{with }s^{\prime} \sim p(\cdot \mid s, a),\nonumber
\end{align}
where $Y\simd U$ denotes equality of probability laws, i.e., the RV $Y$ is distributed according to the same law as $U$.
In order to characterize the contraction properties of these operators, some notion of distance between indexed collections of distributions is necessary.
The \emph{$p$-Wasserstein distance} between two RV $U$ and $Y$ is the $\ell_p$ metric between their inverse cumulative distribution functions (inverse CDFs) \citep{muller1997integral}, i.e.,
$$
d_{p}(U, Y)\equiv\left(\int_{0}^{1}\left|F_{Y}^{-1}(\omega)-F_{U}^{-1}(\omega)\right|^{p} d \omega\right)^{1 / p}
$$
where, for a RV $Y$, the \emph{inverse CDF} $
F_{Y}^{-1}(\omega)\equiv\inf \left\{y \in \mathbb{R}: \omega \leq F_{Y}(y)\right\}
$
where ${F_{Y}(y)\equiv\operatorname{Pr}(Y \leq y)}$ is the CDF of $Y$.\footnote{For $p=\infty$, $d_{\infty}(Y, U)\equiv\sup _{\omega \in[0,1]}\left|F_{Y}^{-1}(\omega)-F_{U}^{-1}(\omega)\right|$.}
Then, the maximal Wasserstein metric between two indexed collections of distributions $Z_1$ and $Z_2$ is defined as
$\bar{d}_{p}\left(Z_{1}, Z_{2}\right)\equiv\sup _{s, a} d_{p}\left(Z_{1}(s, a), Z_{2}(s, a)\right)$.
\cite[Lemma 3]{c51} shows that  $\mathcal{T}^{\pi}$ is a contraction in $\bar{d}_{p}$, i.e., 
\begin{equation}
\bar{d}_{p}\left( \mathcal{T}^{\pi} Z_{1},  \mathcal{T}^{\pi} Z_{2}\right) \leq \gamma \bar{d}_{p}\left(Z_{1}, Z_{2}\right)
.
\end{equation}
The case of the distributional optimality operator $\mathcal{T}$ is more involved. In general, it is not a contraction \citep{c51}. However, based on the fact that $\mathcal{T}^\pi$ is a contraction, \cite{c51} proves that, if the optimal policy is unique, then the iterates $Z_{k+1} \leftarrow \mathcal{T}Z_k$ converge to $Z^{\pi^*}$ (in $p$-Wasserstein metric, $\forall s,a$) and, under some conditions, $\mathcal{T}$ has a unique fixed point corresponding to an optimal value distribution.  

\subsection{Finite support projection}

 Previous approaches of DRL project return distributions $Z(s,a)$ onto a space of distributions of finite support, modeled by a mixture of Diracs over $N$ support points $\theta_i(s,a), i=1..N$, i.e.,
\begin{equation}
    Z_\theta(s,a) \equiv \sum_{i=1}^N p_i(s,a) \delta_{\theta_i(s,a)} 
\end{equation}
which yields a staircase CDF %
$\sum_{i=1}^N  p_i(s,a) \ind_{z\geq \theta_i(s,a)}$.
Different approaches have been followed to parameterize these distributions depending on whether $p_i$ and $\theta_i$ are learned or fixed.
In this work, we consider $p_i$ fixed and $\theta_i$ a learned parameter.

In order to analyze how arbitrary distributions are mapped into these finite representations, different projection operators are defined as minimizers of some distance between distributions.
For instance, in \cite{qrdqn}, the 1-Wasserstein projection $\Pi_{W_{1}}$ is used
and  it is shown that the resulting projected Bellman operator remains a contraction, i.e.,
\begin{equation}
\label{eq:wasserstein-contraction}
\bar{d}_{\infty}\left(\Pi_{W_{1}} \mathcal{T}^{\pi} Z_{1}, \Pi_{W_{1}} \mathcal{T}^{\pi} Z_{2}\right) \leq \gamma \bar{d}_{\infty}\left(Z_{1}, Z_{2}\right)
.
\end{equation}

However, since Wasserstein distances suffer from biased gradients \citep{cramerGAN,c51}, the \emph{quantile regression (QR) loss} is used in practice, guaranteeing the same minimizer and enjoying unbiased gradients \citep{qrdqn}.
Given a target distribution $\TF$, the QR loss, which allows to learn the parameters $\{\theta_1,\dots,\theta_N\}$ of $F(z)\equiv \frac{1}{N}\sum_{i=1}^N  \ind_{z\geq \theta_i}$, is defined as
\begin{align}
\label{eq:qr-loss}
\lqr(F,\TF) &\equiv  \sum_{i=1}^N  \EX_{Z\sim \TF} \left[\rho_{\htaui}(Z-\theta_i)\right] \\
&\text{ with }
\rho_\tau(u) \equiv u(\tau - \ind_{u<0})
\end{align}
where $\htaui$ are the midpoints of a uniform grid of $N$ quantile levels, i.e., $\htaui \equiv \frac{2i-1}{2N}$.
Note that this definition makes $\theta_i$ an estimate of the $\htaui$-quantile. As we shall see in the next section (cf.~Remark \ref{rmk:permutation-invariance}), this correspondence is not enforced by the Cramér projection.
Improved empirical results have been reported in \cite{qrdqn} by Huberizing the QR loss, i.e., by replacing $\rho_\tau(u)$ by
$\rho^\kappa_{\tau}(u)=\abs{\tau-\ind_{u<0}} \mathcal{L}_\kappa(u)  
$ where $\mathcal{L}_\kappa(u)$  is the \emph{Huber loss} \citep{huber1964}
\begin{equation}
\mathcal{L}_{\kappa}(u)\equiv 
\begin{cases}\frac{1}{2} u^{2}, & \text { if }|u| \leq \kappa \\ \kappa\left(|u|-\frac{1}{2} \kappa\right), & \text { otherwise }
\end{cases}.
\end{equation}

\section{CRAMÉR AND 1-WASSERSTEIN PROJECTION EQUIVALENCE}
\label{sec:equiv-wasserstein}
The $\ellp$ distance between two RV $U$ and $Y$ is the $\ell_p$ metric between their CDFs, i.e.,
$$
\ellp(U, Y)\equiv\left(\int_{-\infty}^{\infty}\left|F_{Y}(z)-F_{U}(z)\right|^{p} dz\right)^{1 / p}.
$$
The \emph{Cramér distance} corresponds to the $\ellp$ distance for $p=2$.
We now show that, given an arbitrary distribution and a grid of quantile levels, there is a staircase representation that minimizes the $\ellp$ distance, which puts the quantile values at the inverse of the quantile level midpoints. %
We first introduce an auxiliary Lemma.

\begin{restatable}[]{lem}{auxiliary}
For any $\tau,\tau'\in[0,1]$ with $\tau<\tau'$ and CDF $F$ with inverse $F^{-1}$, let $t\equiv F^{-1}(\tau)$ and $t'\equiv F^{-1}(\tau')$  and  consider the scaled and vertically shifted Heaviside step function $$H_\theta^{\tau,\tau'}(z)\equiv\tau+(\tau'-\tau)\ind_{z\geq \theta}.$$
Then, for any $p\in\R, p>1$,  the set of $\theta\in[t,t']$ minimizing
\begin{equation}
\label{eq:substructure}
\int_{t}^{t'} \lvert F(z) - H_\theta^{\tau,\tau'} \rvert^p dz
\end{equation}
is given by
\begin{equation}
\label{eq:mid-point-sol}
\left\{ \theta \in[t,t'] \vert F(\theta) = \left( \frac{\tau+\tau'}{2} \right) \right\}.
\end{equation}
If $F^{-1}$ is the inverse CDF, then $F^{-1}((\tau+\tau')/2)$ is always a valid minimizer, and if $F^{-1}$ is continuous at $(\tau+\tau')/2$, then $F^{-1}((\tau+\tau')/2)$ is the unique minimizer.
\end{restatable}
\begin{proof}
A visual intuition of the proof is shown in Fig.~\ref{fig:midpoint-intuition}. See Appendix \ref{app:proofs} for details.
\end{proof}

\begin{restatable}[]{thm}{thmminimizer}
\label{thm:equiv-minimizer}
Given $p_i\geq 0, i=1..N$ such that $\sum_i p_i =1$, the $\ell_p$ distance between $F$ and a mixture of Heaviside step functions $F_N(z)=\sum_{i=1}^N p_i \ind_{z\geq \theta_i}$  is minimized with $\theta_i= F^{-1}((\tau_i+\tau_{i-1})/2)$  where  $\tau_i$ are the quantile levels $\tau_i = \sum_{j=1}^i p_j$ and $F^{-1}$ is the inverse CDF.   
\end{restatable}
\begin{proof}

Let $t_i\equiv F^{-1}(\tau_i)$. We first prove that an optimal $\theta^\star$ satisfies $t_{i-1}\leq\theta^\star_i\leq t_i$. See Fig.~\ref{fig:intuition} for an intuition. 

\begin{figure}
\centering
\includegraphics[width=7cm]{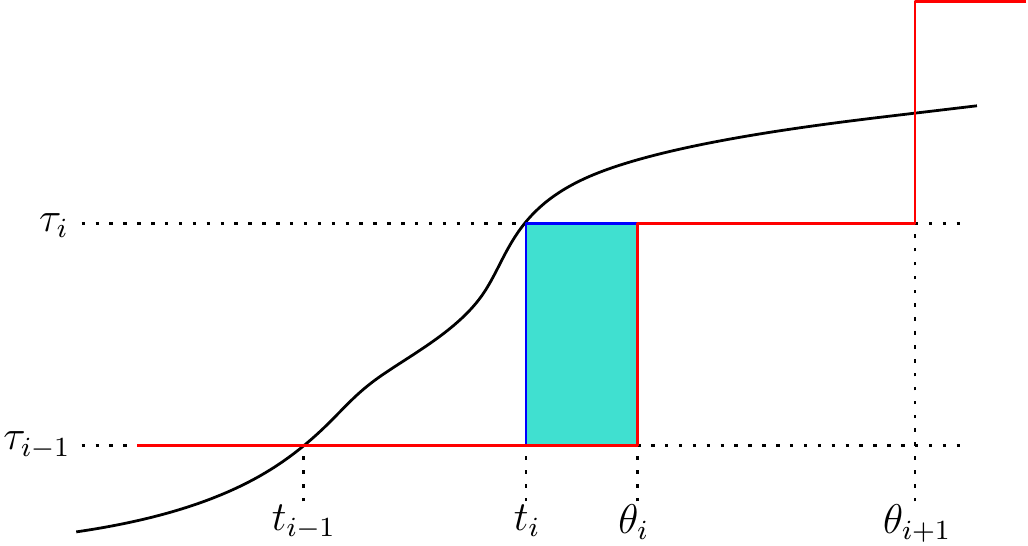}
\includegraphics[width=7cm]{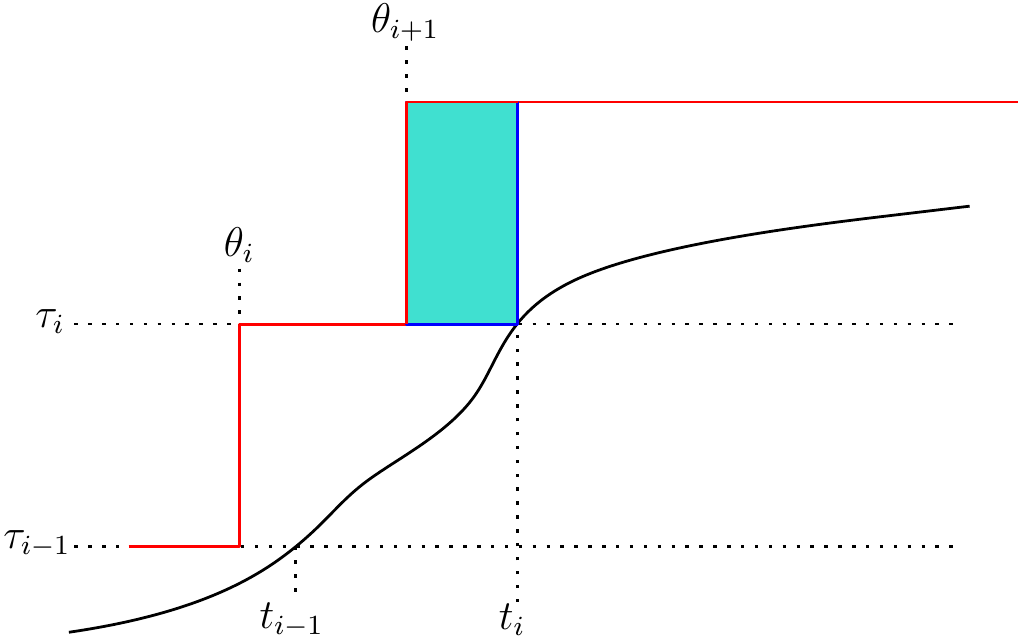}

	\caption{{\bf Intuition for proving $t_{i-1}\leq\theta^\star_i\leq t_i$.} The $\ellp$ distance can be decreased by moving $\theta_i$ to $t_i$, in the first situation, and $\theta_{i+1}$ to $t_i$, in the second one. The shaded area represents the decrease for $p=1$.}
	\label{fig:intuition} 
\end{figure}

Without loss of generality, we assume that $\theta^\star_1\leq\dots\leq\theta^\star_N$.
Let us suppose that there is an optimal $F_N$ with $\theta_1 \geq t_1$.
We can write the $p$-th power of the $\ellp$ distance as
\begin{align}
&\ellpp(F,F_N) = \int_{-\infty}^{t_1} \lvert F(z) -F_N(z) \rvert^p dz 
 \\
&+\int_{t_1}^{\theta_{2}} \lvert F(z) -F_N(z) \rvert^p dz 
+ \int_{\theta_{2}}^\infty \lvert F(z) -F_N(z) \rvert^p dz \nonumber
\end{align}
The value of the middle term strictly decreases when $\theta_1$ decreases toward $t_1$ (while the other terms are unaffected) since
\begin{align}
    &\int_{t_1}^{\theta_{2}} \lvert F(z) -F_N(z) \rvert^p dz 
    = \int_{t_1}^{\theta_{2}} \lvert F(z) -H_{\theta_1}^{0,\tau_1}(z) \rvert^p dz \nonumber \\
    &= \int_{t_1}^{\theta_{1}} F(z)^p dz + \int_{\theta_1}^{\theta_{2}} (F(z) - \tau_1)^p dz
\end{align}
and $F(z)^p>(F(z) - \tau_1)^p$.
In consequence $\theta_1 = t_1$ ; 
It proves that no optimal exist for $\theta_1 > t_1$, and thus that we have $\theta_1 \leq t_1$.

By induction, we assume that $\theta^\star_{n-1} \leq t_{n-1}$.
As before, we suppose, that there is an optimal $F_N$ with $\theta_n \geq t_n$ and we observe that the value of the term 
\begin{align}
    &\int_{t_n}^{\theta_{n+1}} \lvert F(z) -F_N(z) \rvert^p dz  \\
    &= \int_{t_n}^{\theta_{n+1}} \lvert F(z) -H_{\theta_n}^{\tau_{n-1},\tau_{n}}(z) \rvert^p dz  \\
    &= \int_{t_n}^{\theta_n} (F(z)-\tau_{n-1}) ^p dz + \int_{\theta_{n}}^{\theta_{n+1}} (F(z) - \tau_{n})^p dz \nonumber
\end{align}
strictly decreases when $\theta_n$ decreases toward $t_n$ since $(F(z)-\tau_{n-1})^p>(F(z)-\tau_{n})^p$.
In consequence $\theta_n = t_n$ ; 
it proves that no optimal exist for $\theta_n > t_n$, and thus that we have $\theta_n \leq t_n  \forall n\in\{1..N\}$.
Analogously, starting by $\theta_N$ and going backwards, we can prove that $\theta_n\geq t_{n-1} \forall n\in\{1..N\}$. 
This allows us to show that the optimization problem has an optimal substructure and thus it amounts to solving independent minimization problems of the form \eqref{eq:substructure}, i.e., 
\begin{align}
&\min_{\theta_1,\dots,\theta_N} \ellpp(F,F_N) = \min_{\theta_1,\dots,\theta_N} \sum_{i=1}^N  \int_{t_{i-1}}^{t_i} \lvert F(z) -F_N(z) \rvert^p dz \nonumber\\
&= \sum_{i=1}^N \min_{\theta_i} \int_{t_{i-1}}^{t_i} \lvert F(z) -H_{\theta_i}^{\tau_{i-1},\tau_{i}}(z) \rvert^p dz 
\end{align}
with $t_0\equiv-\infty$. 
\end{proof}

\begin{figure}
    \centering
    \includegraphics[width=.48\columnwidth]{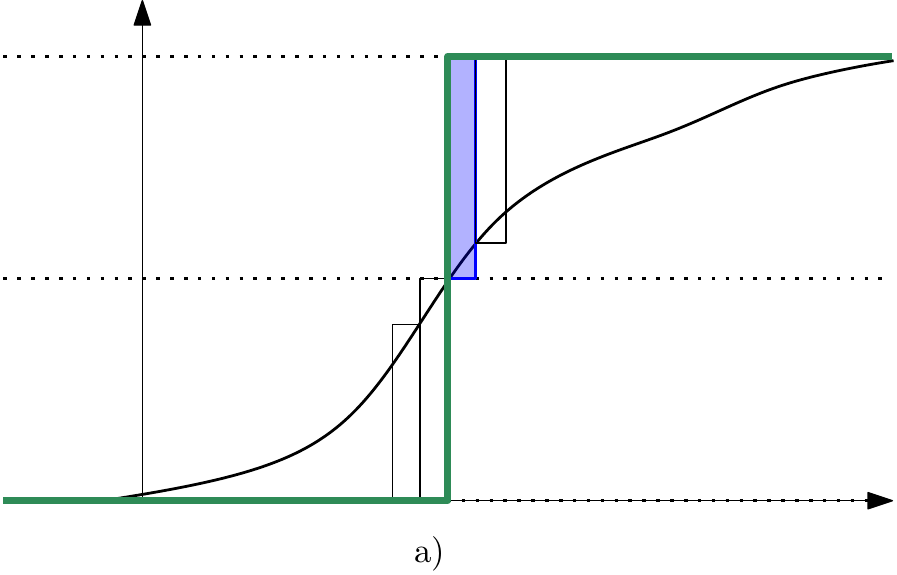}
    \includegraphics[width=.48\columnwidth]{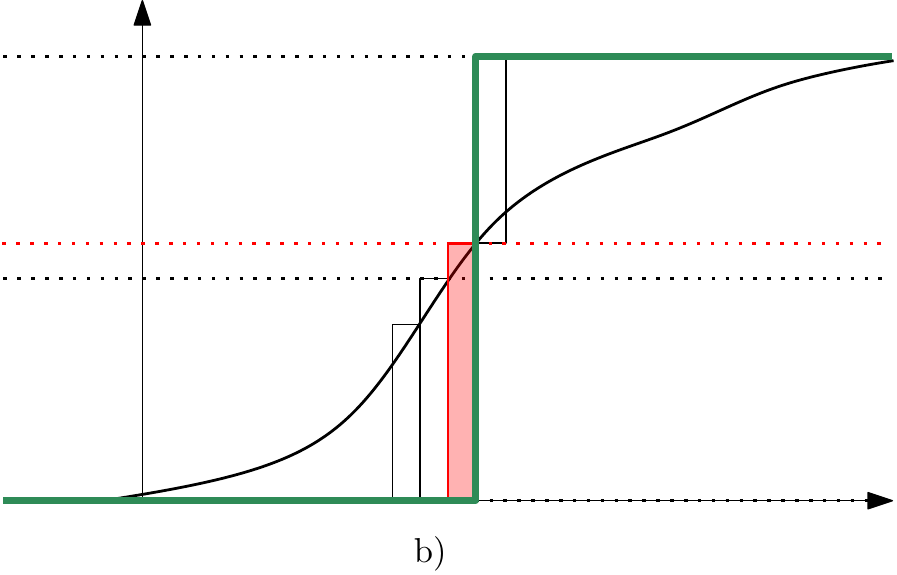}

    \caption{{\bf Midpoint minimizer.} a) The black curve is approximated by one step function (in green) located at the inverse of the mid-point. The rectangles represent an approximation of the $\ellp$ distance. b) If we move the step function to the right, the blue rectangle will be replaced by the larger red one. }
    \label{fig:midpoint-intuition}
\end{figure}

\begin{remark}\label{rmk:permutation-invariance}
For simplicity, we chose $\theta_i= F^{-1}((\tau_i+\tau_{i-1})/2)$, however any permutation $\sigma$ in the symmetric group of size $N$ makes $\tilde{\theta}_i\equiv\theta_{\sigma(i)}$ a minimizer too. %
\end{remark}

We define the $\ellp$ projection of an arbitrary CDF $F$ with inverse CDF $F^{-1}$ onto a grid of quantile levels as 
\begin{equation}
\pip F \equiv F^\star_N(z)=\sum_{i=1}^N p_i \ind_{z\geq \theta^\star_i}    \end{equation}
with $\theta^\star_i= F^{-1}((\tau_i+\tau_{i-1})/2)$. Therefore, it is equivalent to the 1-Wasserstein projection and to QR loss minimization \citep[Lemma 2]{qrdqn}, which implies the following corollary.

\begin{restatable}[]{cor}{contraction}
The Cramér projected distributional Bellman operator is a contraction in $\bar{d}_{\infty}$ i.e.
\begin{equation}
\label{eq:cramer-contraction}
\bar{d}_{\infty}\left(\Pi_{\ellp} \mathcal{T}^{\pi} Z_{1}, \Pi_{\ellp} \mathcal{T}^{\pi} Z_{2}\right) \leq \gamma \bar{d}_{\infty}\left(Z_{1}, Z_{2}\right)
.
\end{equation}
\end{restatable}
\begin{proof}
It follows directly from Eq.~\eqref{eq:wasserstein-contraction} \cite[Lemma 3]{c51} and Theorem \ref{thm:equiv-minimizer}.
\end{proof}

\section{CRAMÉR AND QR LOSS OPTIMIZATION EQUIVALENCE}
\label{sec:equiv-qrloss}

In order to put in evidence the relationship between the gradients of the QR loss and the squared Cramér distance---which we refer to as \emph{Cramér loss}---, we first present an alternative formula for the latter.

\begin{restatable}[]{lem}{lemmaaltexpression}
\label{lem:constructivism}
Given two staircase distributions $F(z)=\frac{1}{N}\sum_{i=1}^N  \ind_{z\geq \theta_i}$ and $\TF(z)=\frac{1}{N}\sum_{i=1}^N  \ind_{z\geq \Ttheta_i}$ such that $\theta_1\leq\dots\leq\theta_N$ and $\Ttheta_1\leq\dots\leq\Ttheta_N$.
Let $\uij\equiv \Ttheta_j -\theta_i$ and 
$\dij\equiv \ind_{\uij<0}$.
The squared Cramér distance between the distributions can be expressed as 
\begin{align}
\label{eq:alt-cramer}
&\int_{-\infty}^\infty (F(z) - \TF(z))^2 dz =\\
&\frac{1}{N^2}  \sum_{i=1}^N\left(  \auii +  \sum_{j=i+1}^N  \dij 2\auij +   \sum_{j=1}^{i-1}   (1-\dij) 2\auij  \right)
. \nonumber
\end{align}
\end{restatable}
\begin{proof}
We compute the squared Cramér distance in a constructive way.
The idea is to cover the area between the two curves with rectangular tiles as in Fig.~\ref{fig:constructivism} to compute the integral by pieces.  A tile of height $i/N$ and width $u$ corresponds to the term $u(i/N)^2$.
We start from %
a) and replace parts of tiles to arrive to b).

\begin{figure}
\centering
\includegraphics[width=.95\columnwidth]{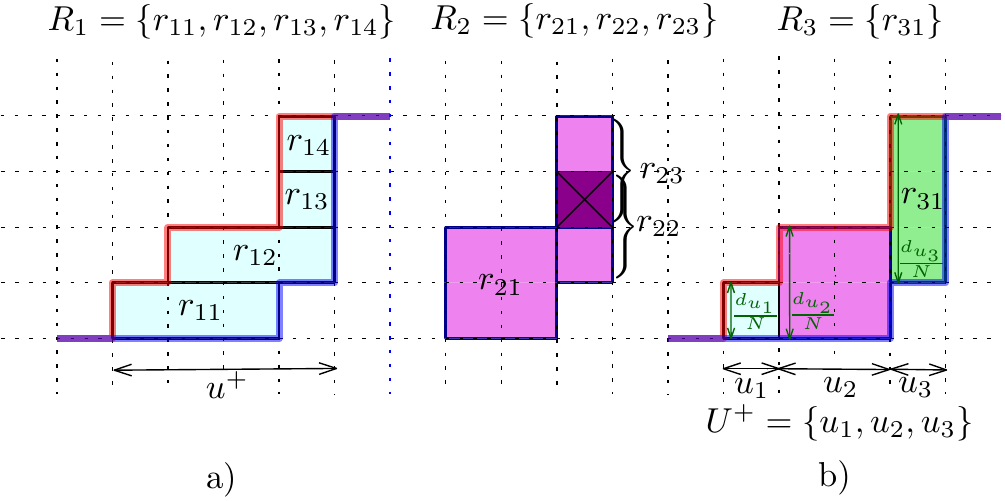}
	\caption{{\bf Computing the Cramér distance between $\TF$ (red) and $F$ (blue) with the tiling operator $T$.} {a) starting point represents $\T_1=\frac{1}{N^2}\sum_{r\in R_1} u_r $}.
	b) ending point represents the squared Cramér distance $\frac{1}{N^2} \left( u_1 1^2 + u_2 2^2 + u_3 3^2\right)$, where $u_i$ is the width of each rectangles in b). Only the leftmost part of $r_{11}$ remains in b), the rest has been replaced by taller rectangles occupying the whole height $\frac{d_{u_i}}{N}$. The middle diagram illustrates the effect of the tiling operator $\T_2$ yielding the final rectangle in the middle and, on the right, two overlapping rectangles---that need to be replaced by a taller one---and an oversubstracted rectangle (with a cross). The result of $\T_1+\T_2+\T_3$ is shown in b), a rectangle of height $3/N$ has been added, the two overlapping rectangles have been removed and the oversubstracted rectangle has been added back. }
	\label{fig:constructivism} 
\end{figure}

First, we formally define our tiling operator $T$. Second, we show that it is well built: the sum of the tiles given by $T$ %
is equal to the squared Cramér distance between the two curves. Third, we derive Eq.~\eqref{eq:alt-cramer} by using that tiling operator.

{\bf (Tiling operator)} First consider an interval $u^+\equiv[t_1,t_2]$ such that $\TF(t_1)=F(t_1)$, $\TF(t_2)=F(t_2)$ and $\TF(z) > F(z) \;\forall z\in(t_1,t_2)$.
Let us define the tiling operator $\T_h$ for  $h\geq1$  
\begin{align}
\label{eq:tiling-op}
    &\T_h(F,\TF, u^+) \nonumber\\ %
&\equiv\frac{1}{N^2}\sum_{r\in R_h} u_r \left( h^2 - 2(h-1)^2 + \ind_{h>1} (h-2)^2 \right)  \\
\label{eq:tiling-op2}
    &=\frac{1}{N^2}\sum_{r\in R_h} u_r(2 - \ind_{h=1})  
\end{align}
where $u_r$ is the width of a rectangle $r$ and $R_h$ is the set of rectangles of height $h/N$ whose upper left and lower right angles are aligned with quantiles of, respectively, $\TF$ and $F$ lying in $u^+$; more formally, 
$R_h\equiv\big\{ r: r \text{ is an axis-parallel rectangle with opposite}\allowbreak\text{corners coordinates }(\theta_i,i/N)\text{ and }(\Ttheta_j,j/N)\;\forall i,j \text{ s.t.}\allowbreak \theta_i,\Ttheta_j\in u^+,  j-i=h \text{ and } \theta_i > \Ttheta_j  \big\}$. Note that these rectangles lie completely within the difference area since $F$ and $\TF$ are monotonically increasing. Note that $\T_1$ corresponds to the initial step depicted in Fig.~\ref{fig:constructivism} a). Intuitively, for $h>1$, Eq.~\eqref{eq:tiling-op} represents the fact that the operator $\T_h$  replaces parts of width $u_r$ of two tiles of height $(h-1)/N$ by a tile of height $h/N$ and width $u_r$ and fixes oversubstracted tiles of the step $h-2$.

{\bf (Soundness)} Let $\T^h (F,\TF, u^+) \equiv \sum_{d=1}^h \T_d(F,\TF, u^+)$.
We are going to express $\T^h$ as a sum over a set $U^+$ of left-closed right-open intervals constituting a partition of $u^+$, s.t.~for any $u\equiv[a,b)\in U^+$ the difference between the CDFs is constant, i.e.,
\begin{equation}
\TF(z)-F(z)=\frac{d_u}{N}>0 \; \forall z\in u,    
\end{equation}
and no quantile lies strictly within $u$, i.e., $\nexists k \text{ s.t. } \theta_k\in(a,b) \lor \Ttheta_k\in(a,b)$. See Fig.~\ref{fig:constructivism} b). We prove by induction the following property.
\begin{align}
\T^h (F,\TF, u^+) &= \label{eq:interval-decomp}\\
&\frac{1}{N^2} \sum_{u\in U^+} \abs{u}\left(\ind_{d_u\leq h}d_u^2 
+ \ind_{d_u>h}g_{u,h} \right) \nonumber
\end{align}
with $g_{u,h}\equiv (d_u-h+1)(2h-1) + (h-1)^2$.
We first express $T_h$ as a sum over $U^+$. We can rearrange the sum in Eq.~\eqref{eq:tiling-op2}, by decomposing each width $u_r$ as a sum of lengths of intervals in $U^+$ and by noting that for each $u\in U^+$ there are $\ind_{d_u\geq h}(d_u-h+1)$ rectangles in $R_h$ with non-empty projection on $u$, as follows 
\begin{align}
\T_h (F,\TF, u^+) &=\\
&\frac{1}{N^2} \sum_{u\in U^+}  \abs{u}\ind_{d_u\geq h}(d_u-h+1) (2 - \ind_{h=1}) \nonumber
\end{align}
In particular, for $h=1$, we have
\begin{equation}
\T_1 (F,\TF, u^+) = \frac{1}{N^2} \sum_{u\in U^+} \abs{u}d_u,
\end{equation}
which validates the base case since $\T^1 (F,\TF, u^+)=\T_1 (F,\TF, u^+)$ and $\ind_{d_u\leq h}d_u^2 + \ind_{d_u>h}g_{u,h}  = d_u$.
We now assume that the property \eqref{eq:interval-decomp} holds for $h-1$ and note that 
$g_{u,h-1}+2(d_u-h+1)=g_{u,h}$.
Then, for $h>1$, 
\begin{align}
\T^{h} (F,\TF, u^+)=&\T^{h-1} (F,\TF, u^+) + \T_h (F,\TF, u^+) \\
=\frac{1}{N^2} \sum_{u\in U^+}\abs{u} &\left( \ind_{d_u\leq h-1}d_u^2  +   \ind_{d_u>h-1} g_{u,h-1} %
\right.\\
&+\left. \ind_{d_u\geq h}2(d_u-h+1) %
\right) \notag\\
= \frac{1}{N^2} \sum_{u\in U^+} \abs{u}&\left(\ind_{d_u\leq h-1}d_u^2 + \ind_{d_u\geq h}
g_{u,h} %
\right) \\
=\frac{1}{N^2} \sum_{u\in U^+} \abs{u}&\left(\ind_{d_u\leq h}d_u^2 + \ind_{d_u> h}
g_{u,h} %
\right)
\end{align}
since $\ind_{d_u>h-1}=\ind_{d_u\geq h}$ and  $\ind_{d_u=h}
g_{u,h} %
= \ind_{d_u=h}d_u^2$.

Since $\ind_{d_u\leq N}=1- \ind_{d_u> N} = 1$, the final tiling $\T^N (F,\TF,u^+)$  corresponds to the squared Cramér distance on the interval $u^+$, i.e.,
\begin{equation}
\T^N (F,\TF, u^+) = \frac{1}{N^2} \sum_{u\in U^+} \abs{u}d_u^2 
.
\end{equation}

{\bf (Final derivation)} Now, we are going to use \eqref{eq:tiling-op2} to get to the claimed expression. First  note that for a rectangle $r\in R_h$ with upper leftmost and lower rightmost angles corresponding, respectively, to $\Ttheta_j$ and $\theta_i$, its width is $u_r = \auij$. Since $\theta_1\leq\dots\leq\theta_N$ and $\Ttheta_1\leq\dots\leq\Ttheta_N$,
when $\TF(z)> F(z)$, each rectangle in $R_h$ corresponds to exactly one pair $(\Ttheta_j,\theta_i)$ such that $(\dij=1)\land(i\leq j)$.
By symmetry, the condition $(\dij=0)\land(j\leq i)$ allows us to consider intervals such that $\TF(z)< F(z)$. This allows to express the sum \eqref{eq:tiling-op2} as sums over indices $i,j$.   
We consider the case $i=j$ separately to avoid double counting and also because it corresponds to $h=1$.
Therefore, from \eqref{eq:tiling-op2}, we have
\begin{align}
\T^N (F,\TF,\R) = \frac{1}{N^2}& \left( \sum_{r\in R_1} u_r + \sum_{h=2}^N\sum_{r\in R_h} 2u_r \right)\\
= \frac{1}{N^2} \left( \sum_{i=1}^N \auii \right.+&\sum_{i=1}^{N-1} \sum_{j=i+1}^N  \dij 2\auij\\  
+&\left.\sum_{j=1}^{N-1} \sum_{i=j+1}^N  (1-\dij) 2\auij \right)
.
\end{align}
By rearranging the sums,
we get  Equation \eqref{eq:alt-cramer}.
\end{proof}

\begin{restatable}[]{cor}{corollarygradients}
\label{cor:col-grad}
Under the conditions of Lemma \ref{lem:constructivism},
\begin{align}
&\frac{\partial \lqr(F,\TF) }{\partial \theta_i} =  \frac{1}{N}\left(\frac{1-2i}{2} + \sum_{j=1}^N \delta_{ij} \right) \\
&\frac{\partial \ell_2^2(F,\TF) }{\partial \theta_i} =  \frac{1}{N^2} \left(1-2i + 2\sum_{j=1}^N \delta_{ij} \right)
.
\end{align}
Therefore, their gradients are collinear, i.e.
\begin{equation}
\nabla_\vtheta \lqr  = \frac{N}{2} \nabla_\vtheta \ell_2^2
.
\end{equation}
\end{restatable}
\begin{proof}
For a target distribution $\TF(z)=\frac{1}{N}\sum_{i=1}^N  \ind_{z\geq \Ttheta_i}$, 
the quantile regression loss can be expressed as
\begin{align}
\lqr(F,\TF) &= \sum_{i=1}^N \frac{1}{N} \sum_{j=1}^N \rho_{\htaui}(\Ttheta_j-\theta_i) \\
&= \frac{1}{N} \sum_{i=1}^N \sum_{j=1}^N  (\Ttheta_j-\theta_i)(\htaui - \dij)
\end{align}
and thus
\begin{align}
\frac{\partial \lqr(F,\TF) }{\partial \theta_i}  &= \frac{1}{N} \sum_{j=1}^N  (\dij - \htaui ) \\
&= \frac{1}{N} \left( \frac{1-2i}{2} + \sum_{j=1}^N  \dij  \right)
.
\end{align}
In order to obtain the partial derivative of the squared Cramér distance, first note that 
$\dij\auij=\dij(\theta_i-\Ttheta_j)$, $(1-\dij)\auij=(1-\dij)(\Ttheta_j-\theta_i)$ and $\auii=\dii(\theta_i-\Ttheta_i)+(1-\dii)(\Ttheta_i-\theta_i)$. By replacing these quantities in \eqref{eq:alt-cramer} and taking the derivative with respect to $\theta_i$ we obtain 
\begin{align}
&\frac{\partial \ell_2^2(F,\TF) }{\partial \theta_i} \\
&= \frac{1}{N^2} \left[2\dii - 1 +  2\left( \sum_{j=i+1}^N \dij + \sum_{j=1}^{i-1} (\dij-1)  \right) \right] \nonumber\\
&= \frac{1}{N^2} \left(2\sum_{j=1}^N \dij   -1 + 2\sum_{j=1}^{i-1} (-1)  \right)\\
&= \frac{1}{N^2} \left(1-2i + 2\sum_{j=1}^N \delta_{ij} \right)
.
\end{align}
\end{proof}
\begin{remark}
\label{rmk:adam-invariance}
Therefore, gradient descent methods whose parameter updates are invariant to 
rescaling of the gradient like ADAM \cite{adam}, yield the same optimization path with both losses.
\end{remark}
\begin{remark}
Huberization of the QR loss breaks the equivalence with the Cramér loss. 
\end{remark}

\section{ALGORITHM}
\label{sec:algorithm}

Formula \eqref{eq:alt-cramer} allows to compute the squared Cramér distance between two staircase distributions $F(z)=\frac{1}{N}\sum_{i=1}^N  \ind_{z\geq \theta_i}$ and $\TF(z)=\frac{1}{N}\sum_{i=1}^N  \ind_{z\geq \Ttheta_i}$ assuming the quantiles are ordered, i.e., $\theta_1\leq\dots\leq\theta_N$ and $\Ttheta_1\leq\dots\leq\Ttheta_N$. That formula involves two nested sums making it of quadratic complexity in $N$ as the quantile regression loss. 
Alternatively, if we consider the sorted sequence of merged quantiles $\vtheta'\equiv \mathrm{sort}\left(\{\theta_i\}_{i=1..N}\bigcup\{\Ttheta_i\}_{i=1..N}\right)$, we have that $F(z) - \TF(z)$ is constant between any two consecutive quantile values of $\vtheta'$ and the difference can be obtained by accumulating the increments from $F$ and the decrements from $\TF$, see Appendix \ref{app:algo} for an illustration and a formal proof.
Therefore, we can express the Cramér loss between two staircase distributions  as follows
\begin{align}
&\int_{-\infty}^\infty (F(z) - \TF(z))^2 dz =\\
&\sum_{i=1}^{2N-1}  \left( \theta'_{i+1} - \theta'_{i}  \right)\left( \sum_{j \text{ s.t. } \theta_j\leq\theta'_i} \frac{1}{N} - \sum_{j \text{ s.t. } \Ttheta_j\leq\theta'_i} \frac{1}{N} \right)^2    \nonumber   
\end{align}
where $\theta'_{i}$ is the $i$-the element of $\vtheta'$.
Algorithm \ref{algo} implements this formula based on sorting the merged quantiles of both distributions, yielding $O(N \log N)$ complexity. Note that this algorithm does not require the input vectors $\vtheta$ and $\vTtheta$ to be ordered.
This has an important consequence on the network that outputs $\vtheta$, since it is not required to be in a particular order as for the QR loss. This permutation equivalence creates symmetries in the loss landscape (see Fig.~\ref{fig:symmetries} in Appendix \ref{algo}, for an illustration). Non-crossing architectures like \cite{nc-qrdqn,ndqfn} eliminate these symmetries by enforcing monotonicity on the output.

\begin{algorithm}
\SetAlgoLined
\DontPrintSemicolon
\KwIn{$\vtheta\equiv[\theta_1,\dots,\theta_N]$, $\vTtheta\equiv[\Ttheta_1,\dots,\Ttheta_N] $: array}
\KwOut{$\int_{-\infty}^\infty (F(z) - \TF(z))^2 dz $ 
}
\BlankLine
\Indp $\vtheta'\leftarrow\mathrm{concat}(\vtheta,\vTtheta)$\;
$i_1,\dots,i_{2N}\leftarrow\mathrm{argsort}(\vtheta') $\;
$\vtheta'\leftarrow\vtheta'[i_1,\dots,i_{2N}]$\;
$\Delta_z\leftarrow\vtheta'[1:] - \vtheta'[:\text{-}1]$\;
$\Delta_\tau\leftarrow \mathrm{concat}\left(-\frac{1}{N}\mathbf{1}_N,\frac{1}{N}\mathbf{1}_N\right)$ \;
$\Delta_\tau\leftarrow \Delta_\tau[i_1,\dots,i_{2N}]$ \;
$\Delta_\tau\leftarrow \mathrm{cumsum}\left(\Delta_\tau\right)[:\text{-}1]$ \;
$I\leftarrow \Delta_\tau * \Delta_\tau * \Delta_z $ \;
\Return $\mathrm{sum}(I)$ 
\caption{{\bf Cramér loss.} The operators $[1:]$ and $[:\text{-}1]$ remove, respectively, the first and the last elements of the array. $\mathbf{1}_N$ denotes an array of $N$ ones and $*$ denotes elementwise multiplication.
 }
\label{algo}
\end{algorithm}

\section{CRAMÉR TD-LEARNING ON SAMPLED TRANSITIONS}
\label{sec:td-learning}
In order to train a DRL agent using the Cramér loss, we extend temporal-difference (TD) learning to distributions. For this, we express distributional Bellman's equations in the language of distributions as in \cite{rowland2018analysis}.
Given a probability distribution $\nu \in$ $\mathscr{P}(\mathbb{R})$ and a measurable function $f: \mathbb{R} \rightarrow \mathbb{R}$, the \emph{push-forward measure} $f_{\#} \nu \in \mathscr{P}(\mathbb{R})$ is defined by $f_{\#} \nu(A)\equiv$ $\nu\left(f^{-1}(A)\right)$, for all Borel sets $A \subseteq \mathbb{R}$.
Let $f_{r, \gamma}(x)\equiv r + \gamma x$ and %
$\eta_{\pi}$ be the collection of return distributions for each state and action, associated with a policy $\pi$. The basis of DRL is given by the fixed point equation
$$\eta_{\pi}(s, a)=\left(\mathcal{T}^{\pi} \eta_{\pi}\right)(s, a) \quad \forall(s, a) \in \mathcal{S} \times \mathcal{A}$$
where $\mathcal{T}^{\pi}: \mathscr{P}(\mathbb{R})^{\mathcal{S} \times \mathcal{A}} \rightarrow \mathscr{P}(\mathbb{R})^{\mathcal{S} \times \mathcal{A}}$ is the distributional Bellman operator on distributions\footnote{Eq.~\eqref{eq:drl-bellman-op} is expressed in the language of random variables.}
defined as
$$\left(\mathcal{T}^{\pi} \eta\right)(s, a) 
\equiv \mathbb{E}_{r,s',a'|s,a} \left(f_{r, \gamma}\right)_{\#} {\eta} \left(s^{\prime}, a^{\prime}\right)$$
for all $\eta \in \mathscr{P}(\mathbb{R})^{\mathcal{X} \times \mathcal{A}}$. 
For Cramér-based TD-learning, we should use a parametric distribution $\eta_\theta$ and a frozen version of it that we call $\eta'$ and do stochastic gradient descent by approximating
$\mathbb{E}_{s,a} \nabla_\theta \ell_2^2\left( \eta_\theta,\mathbb{E}_{r,s',a'|s,a} \left(f_{r, \gamma}\right)_{\#} {\eta'} \left(s^{\prime}, a^{\prime}\right)  \right)$.
Let $F_{\theta}$ and $F_{r,s',a'}$ denote the CDFs of $\eta_\theta$ and $\left(f_{r, \gamma}\right)_{\#} {\eta'} \left(s^{\prime}, a^{\prime}\right)$, respectively.
Following the steps of the proof of  \cite[Theorem 2]{cramerGAN} (\emph{unbiased gradients}):
\begin{align}
&\mathbb{E}_{s,a} \nabla_\theta \ell_2^2\left( F_{\theta},\mathbb{E}_{r,s',a'|s,a} F_{r,s',a'} \right)\nonumber\\
&\stackrel{\mathmakebox[\widthof{=}]{\text{(a)}}}{=} \mathbb{E}_{s,a}\int_{-\infty}^{\infty} \nabla_\theta \left( F_{\theta}(x) - \mathbb{E}_{r,s',a'|s,a} F_{r,s',a'}(x)  \right)^2 dx \nonumber\\
&\stackrel{\mathmakebox[\widthof{=}]{\text{(b)}}}{=} \mathbb{E}_{s,a} \mathbb{E}_{r,s',a'|s,a} \int_{-\infty}^{\infty} 2  \left( F_{\theta}(x) -  F_{r,s',a'}(x)  \right) \nabla_\theta  F_{\theta}(x) dx \nonumber\\
&= \mathbb{E}_{s,a,r,s',a'} \nabla_\theta \ell_2^2\left( F_{\theta}, F_{r,s',a'} \right) \label{eq:exp-gradient}
\end{align}
where (a) and (b) hold assuming that $F_{\theta}$ and $F_{r,s',a'}$ have light enough tails (which is our case since they are mixtures of $N$ Heaviside functions) to avoid infinite squared Cramér distances and expected gradients.
In the control case, the expectation over $a'$ is not needed anymore, since $a'$ is deterministically given by the policy.
Practically, Eq.~\eqref{eq:exp-gradient} allows us to use the average gradient of $\ell_2^2\left( F_{\theta}, F_{r,s',a'} \right)$ over batches of sample transitions for Cramér TD-learning.

\color{black}
\section{EXPERIMENTS}
\label{sec:experiments}

In light of the previous results, we investigate how the differences between the Cramér and the QR losses affect the results in synthetic and Atari 2600 experiments, considering the presence or not of non-crossing constraints and Huberization. The code and the full output of the experiments are available at \url{https://github.com/alherit/cr-dqn}.

\subsection{Synthetic experiment}

We first propose an experiment that is simple but representative of the challenges that DRL faces. 
We consider an MDP with only one possible action at one state $s$ that can transition to two possible states $s_1$ and $s_2$ with probabilities $2/3$ and $1/3$, respectively, each with a different return distribution---a Dirac located at -1 and 1 respectively. The goal is to learn the return distribution at $s$, i.e., the mixture distribution shown in red in Fig.~\ref{fig:synthetic}.
The figure shows the estimated distributions obtained after 1000 training iterations with the different losses and two architectures: a fully connected (FC) one as in QR-DQN and the non-crossing (NC) one of NC-QR-DQN with a comparable number of parameters (2712 and 2702, respectively). 
The networks output $N=12$ quantiles, allowing to represent the mixture exactly. We repeat the experiment 100 times. We show the average 1-Wasserstein distance $d_1$ and the standard deviation to quantify how close are the learned distributions with respect to the true target. See Appendix \ref{app:experiments} for details.

\begin{figure}
\captionsetup[sub]{font=tiny,labelfont={bf}}
\centering
     \begin{subfigure}[b]{0.48\columnwidth}
        \centering
        \includegraphics[align=c,width=\columnwidth]{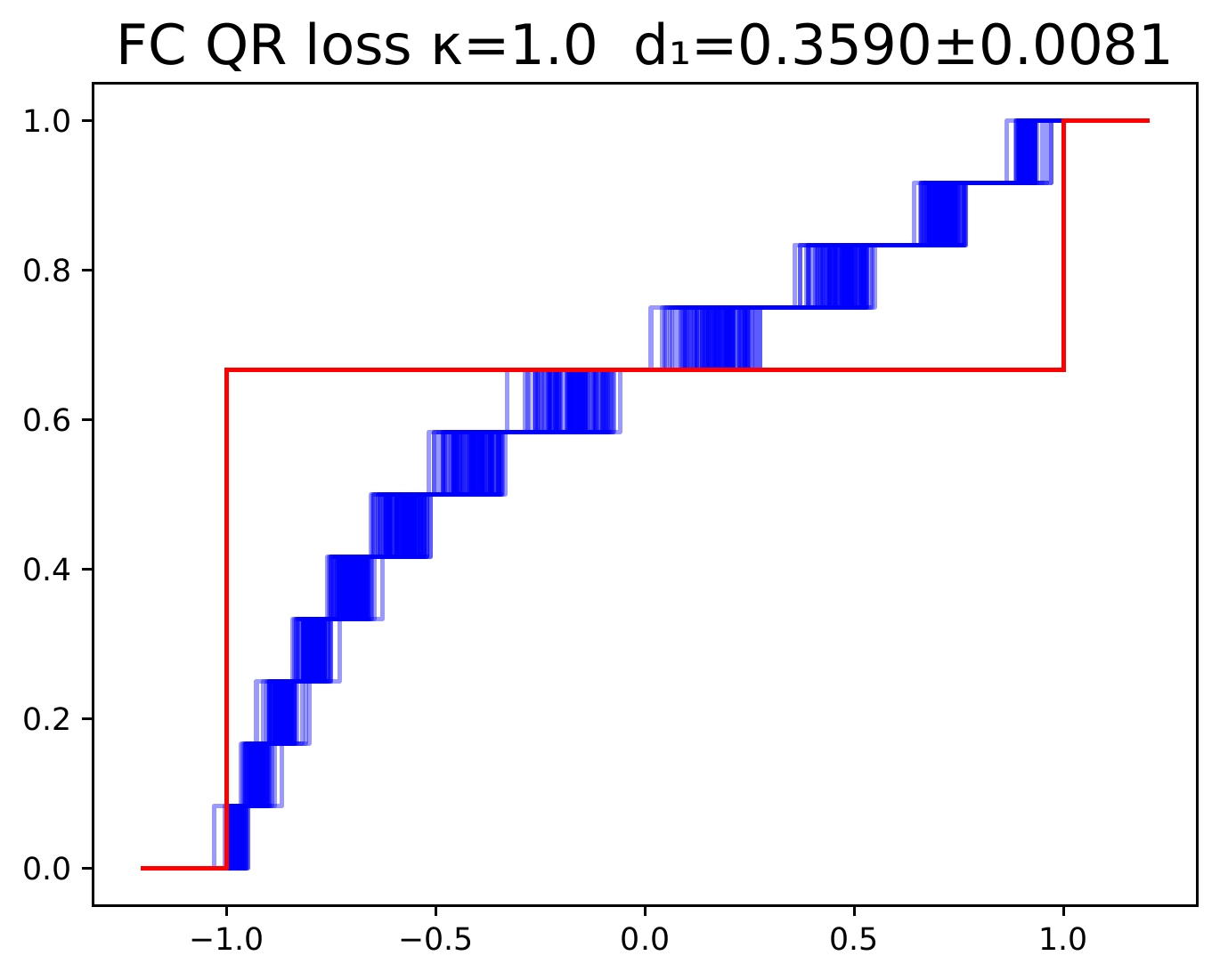}
        \vspace*{-3mm}
        \caption{}
        \label{subfig:fc-qr-loss1.0}
     \end{subfigure}
     \hfill
     \begin{subfigure}[b]{0.48\columnwidth}
        \centering
        \includegraphics[align=c,width=\columnwidth]{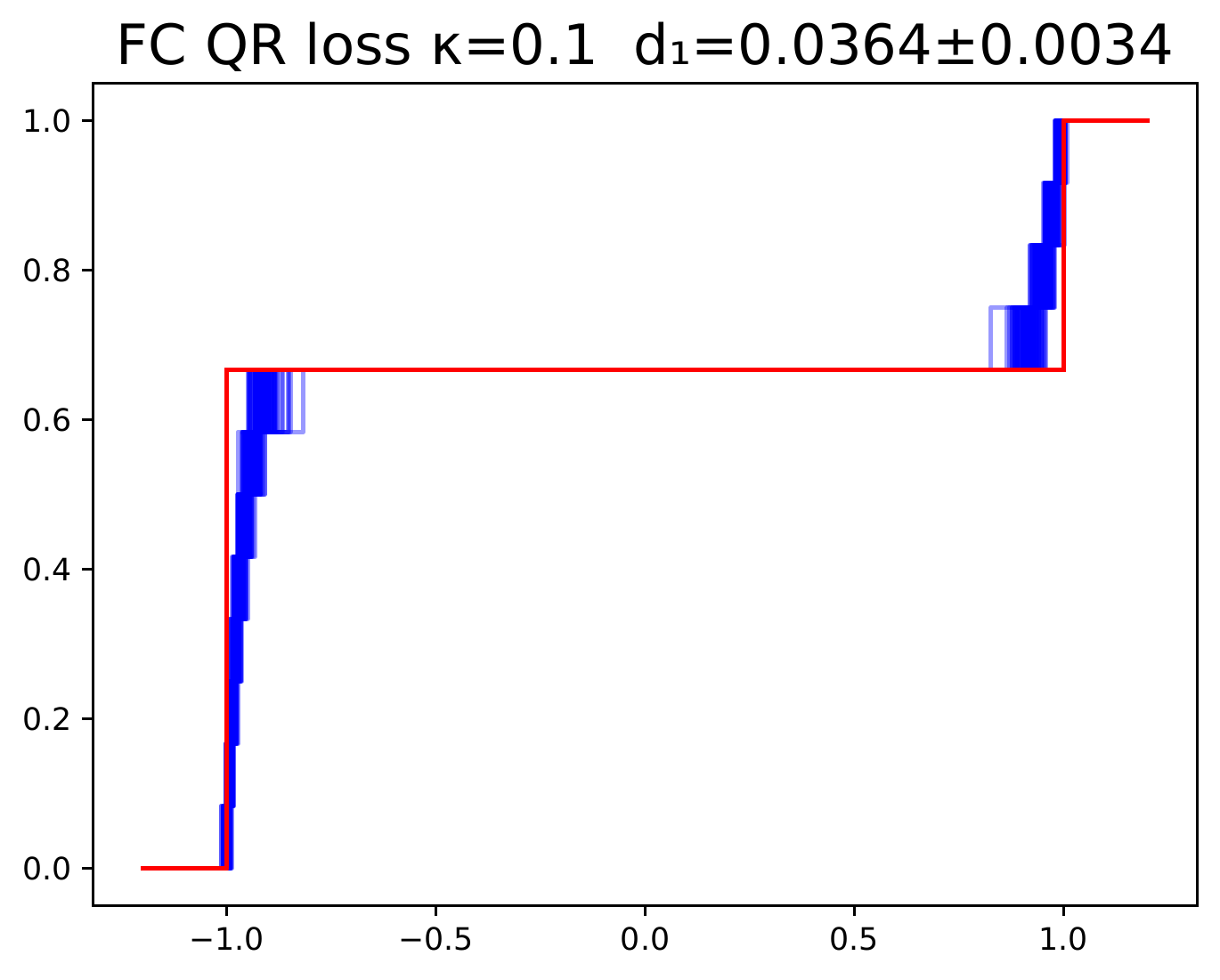}
        \vspace*{-3mm}
        \caption{}
        \label{subfig:fc-qr-loss0.1}
     \end{subfigure}
    \\
    \begin{subfigure}[b]{0.48\columnwidth}
        \centering
        \includegraphics[align=c,width=\columnwidth]{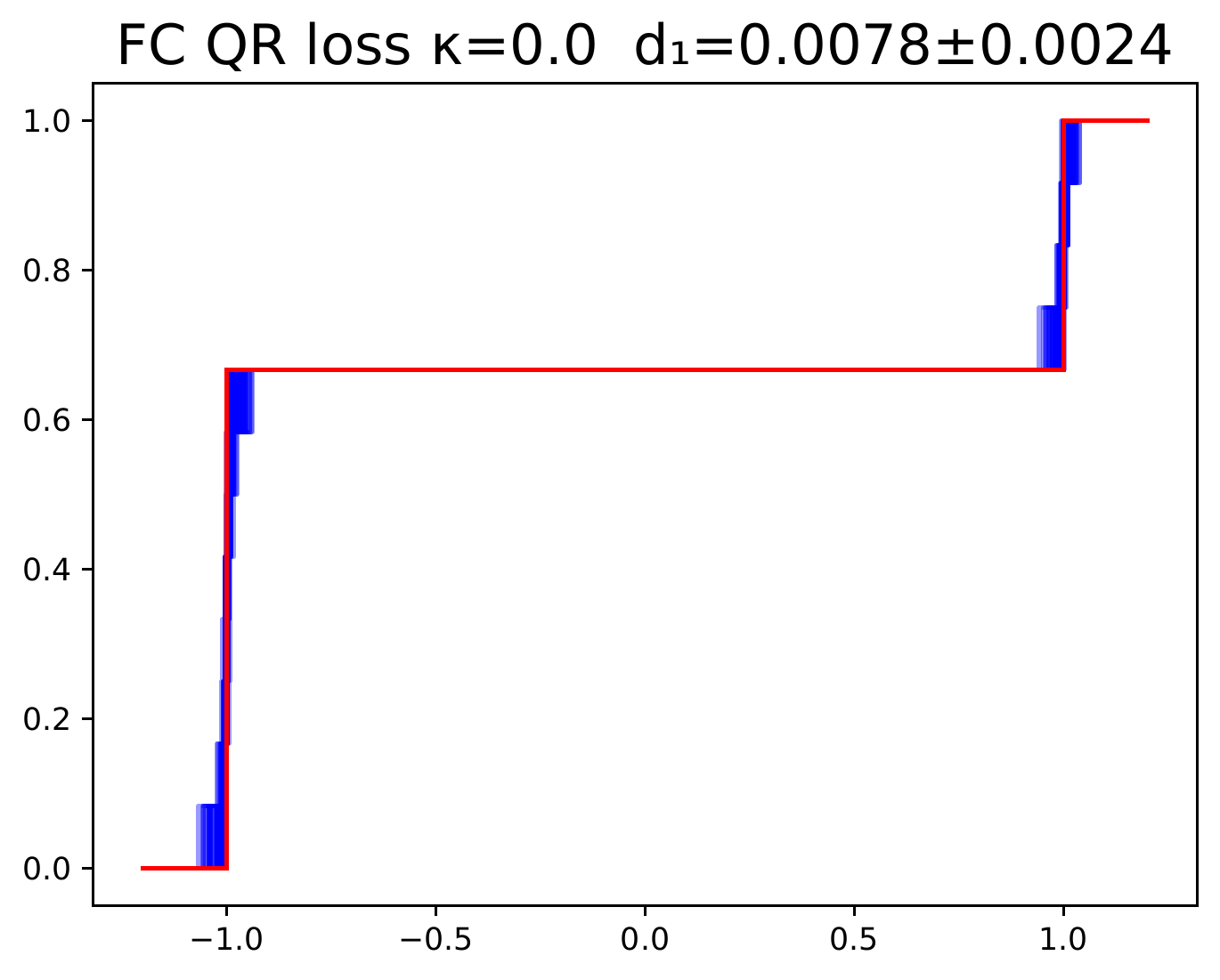}
        \vspace*{-3mm}
        \caption{}
        \label{subfig:fc-qr-loss0.0}
     \end{subfigure}
     \hfill
     \begin{subfigure}[b]{0.48\columnwidth}
        \centering
        \includegraphics[align=c,width=\columnwidth]{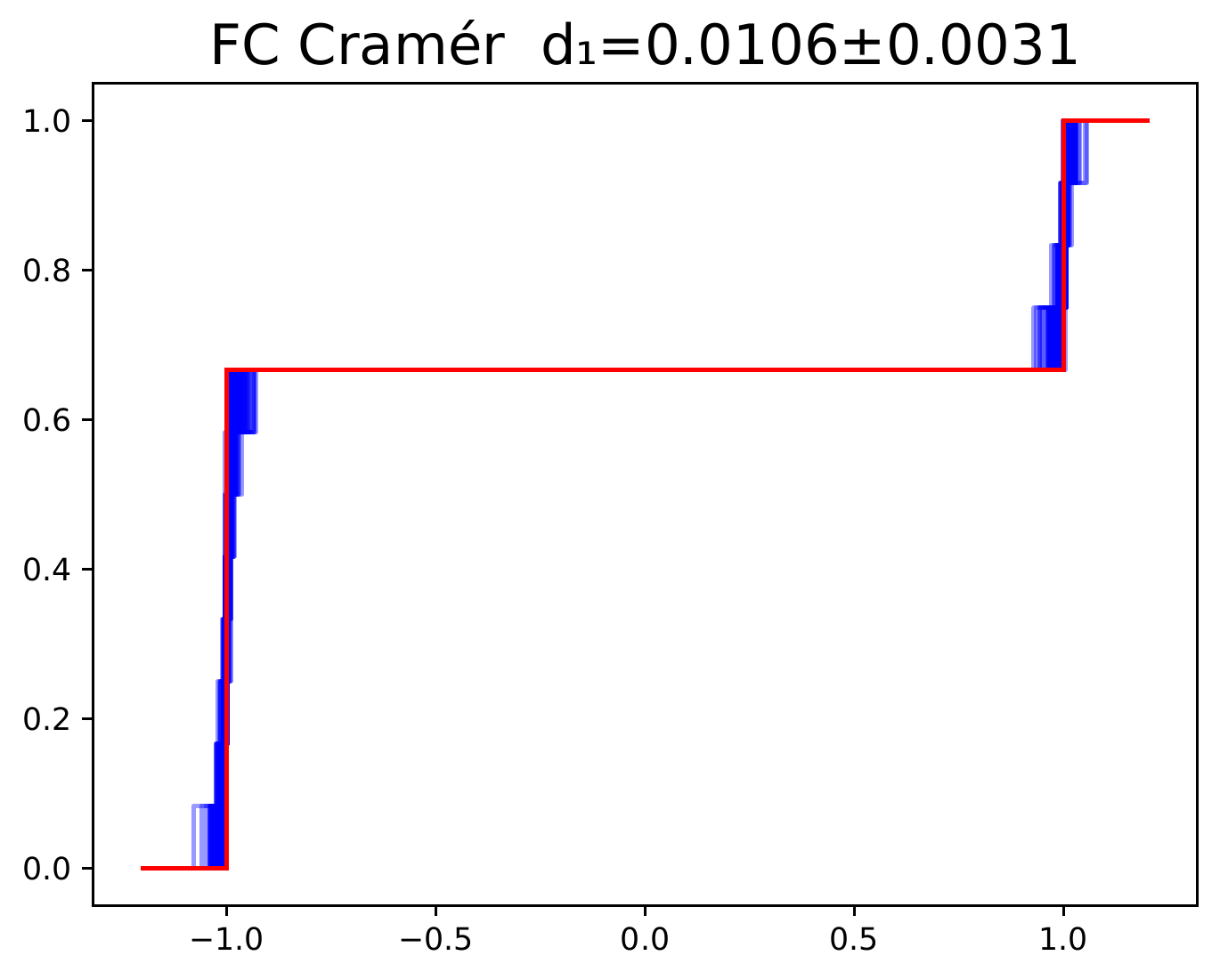}
        \vspace*{-3mm}
        \caption{}
        \label{subfig:fc-cramer}
     \end{subfigure}
     \\
    \begin{subfigure}[b]{0.48\columnwidth}
        \centering
        \includegraphics[align=c,width=\columnwidth]{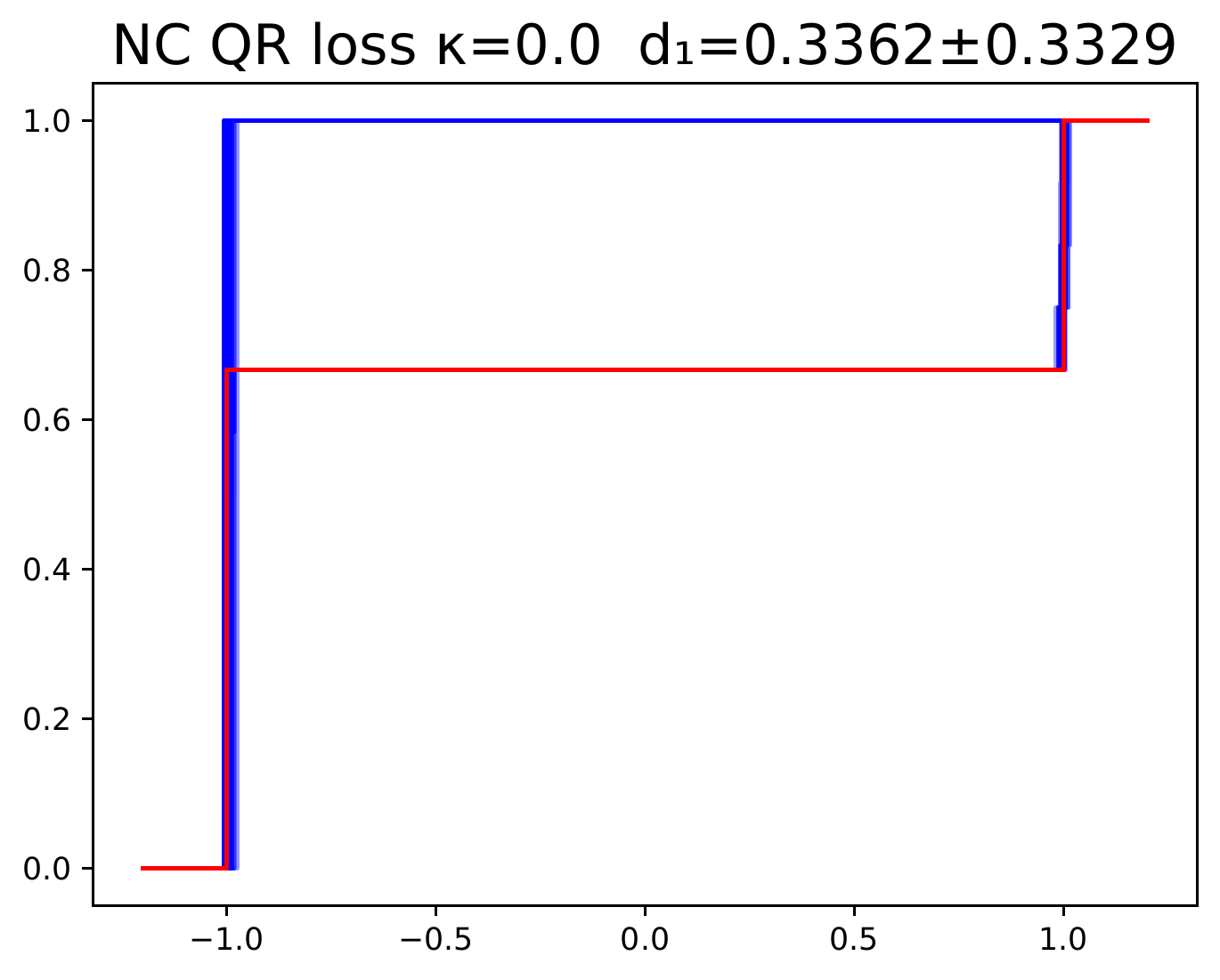}
        \vspace*{-3mm}
        \caption{}
        \label{subfig:nc-qr-loss0.1}
     \end{subfigure}
     \hfill
     \begin{subfigure}[b]{0.48\columnwidth}
        \centering
        \includegraphics[align=c,width=\columnwidth]{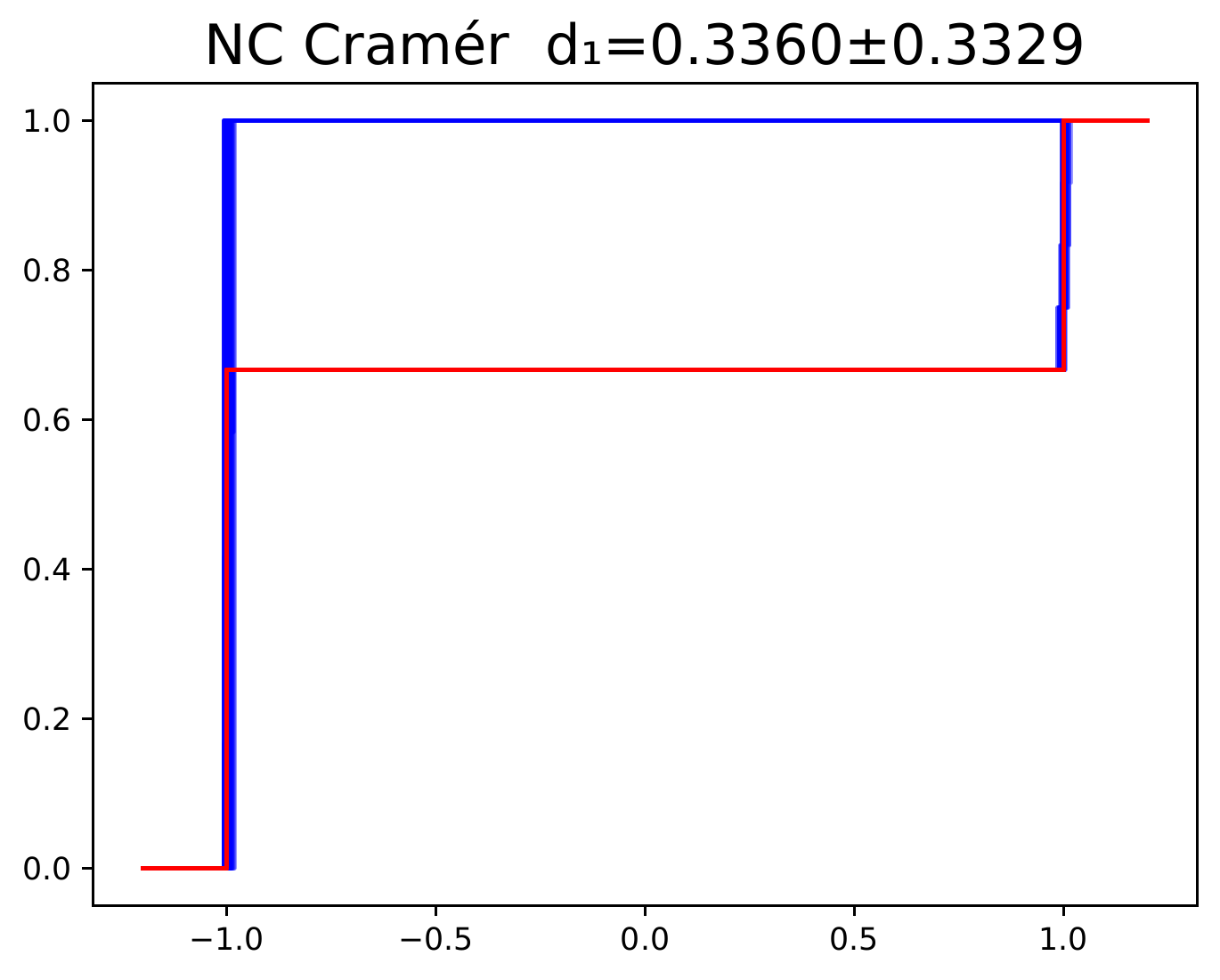}
        \vspace*{-3mm}
        \caption{}
        \label{subfig:nc-cramer}
     \end{subfigure}
     \\
    \begin{subfigure}[b]{0.48\columnwidth}
        \centering
        \includegraphics[align=c,width=\columnwidth]{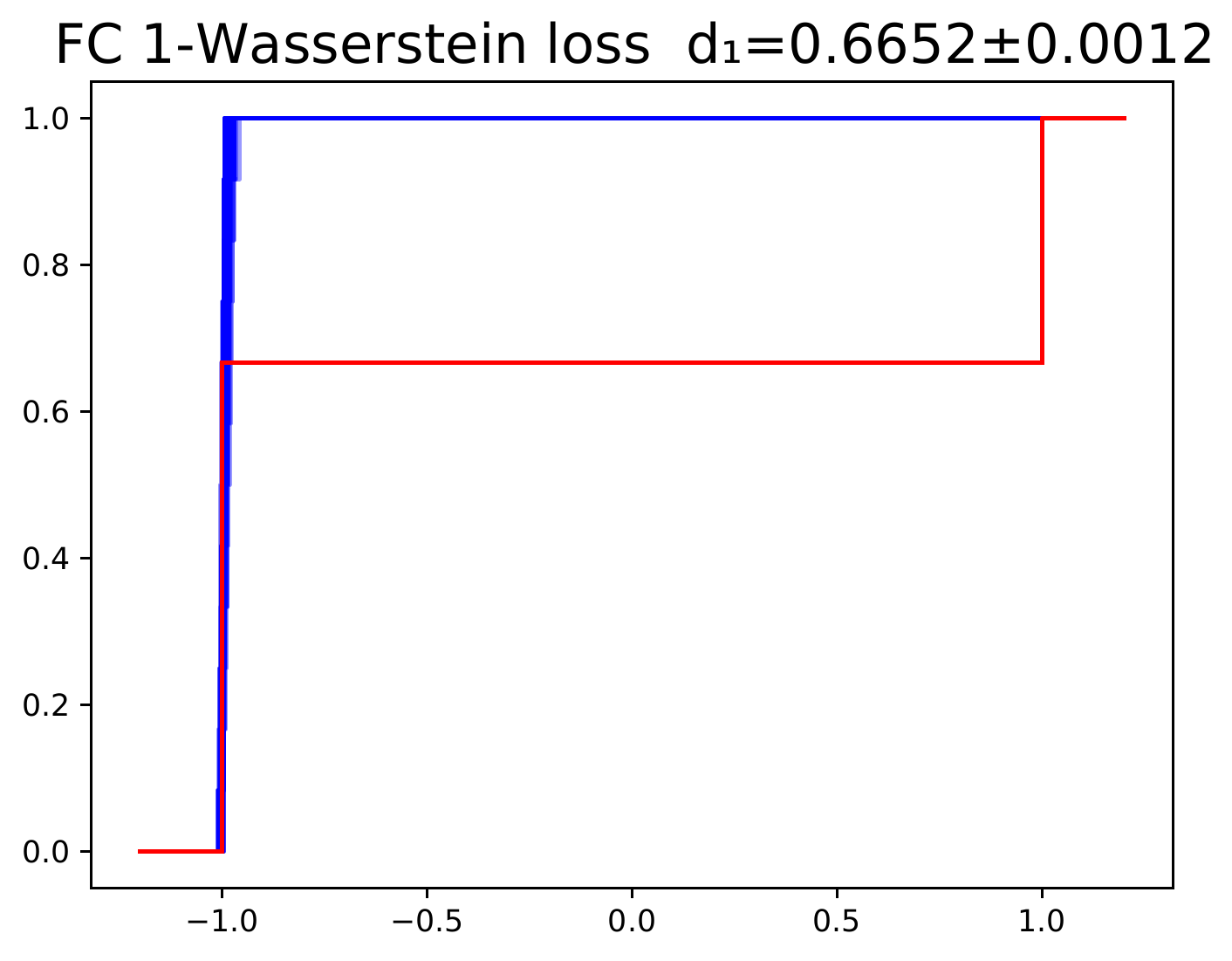}
        \vspace*{-3mm}
        \caption{}
        \label{subfig:fc-wasserstein}
     \end{subfigure}
     \hfill
     \begin{subfigure}[b]{0.48\columnwidth}
        \centering
        \includegraphics[align=c,width=\columnwidth]{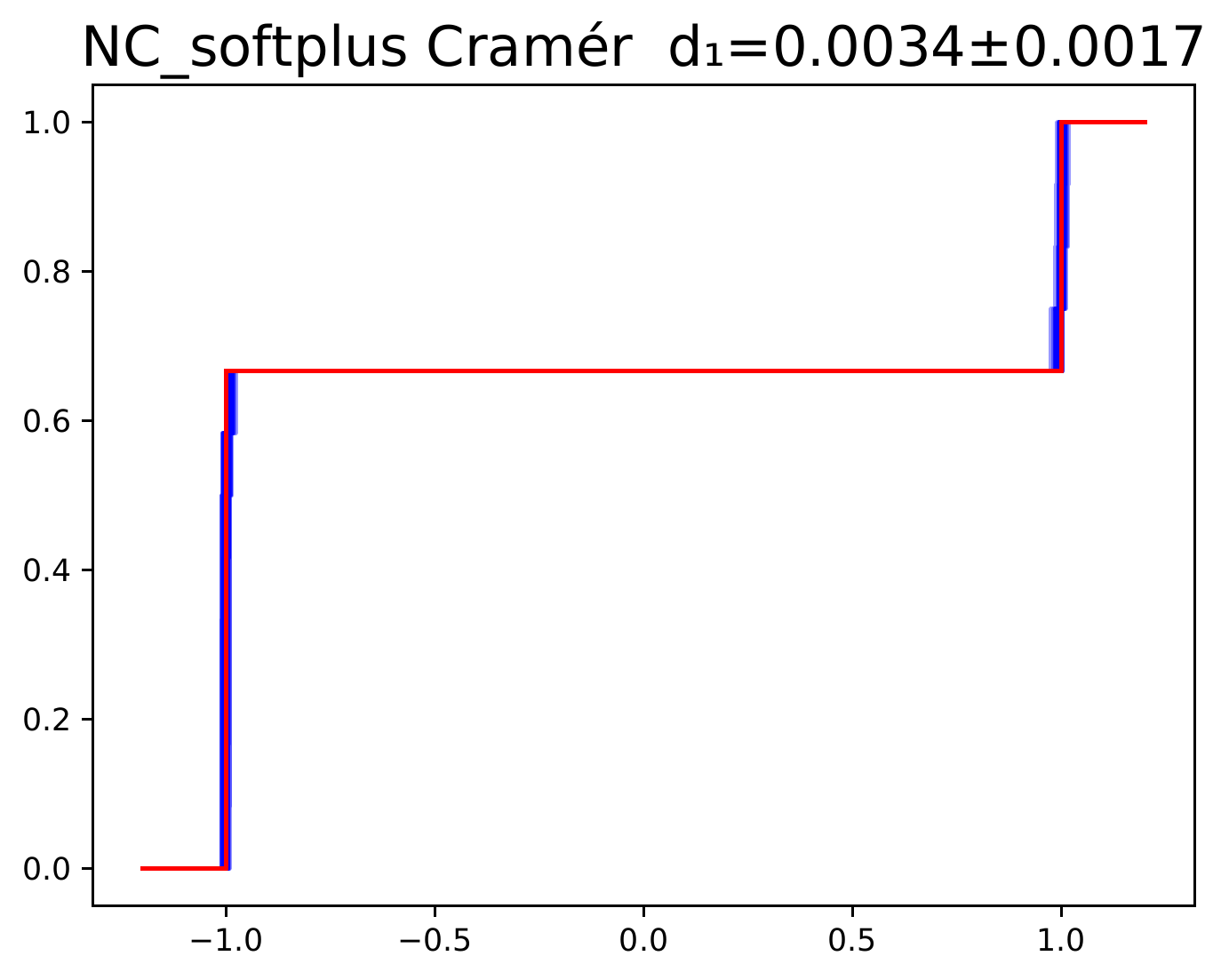}
        \vspace*{-3mm}
        \caption{}
        \label{subfig:nc-softplus-cramer}
     \end{subfigure}

\caption{{\bf Synthetic experiments.} The learned CDF for each trial is shown in blue. The average $d_1$ with the mixture of targets (in red) is shown for each case. 
 }
\label{fig:synthetic}
\end{figure} 
\begin{figure}
\centering
\includegraphics[align=c,width=.48\columnwidth]{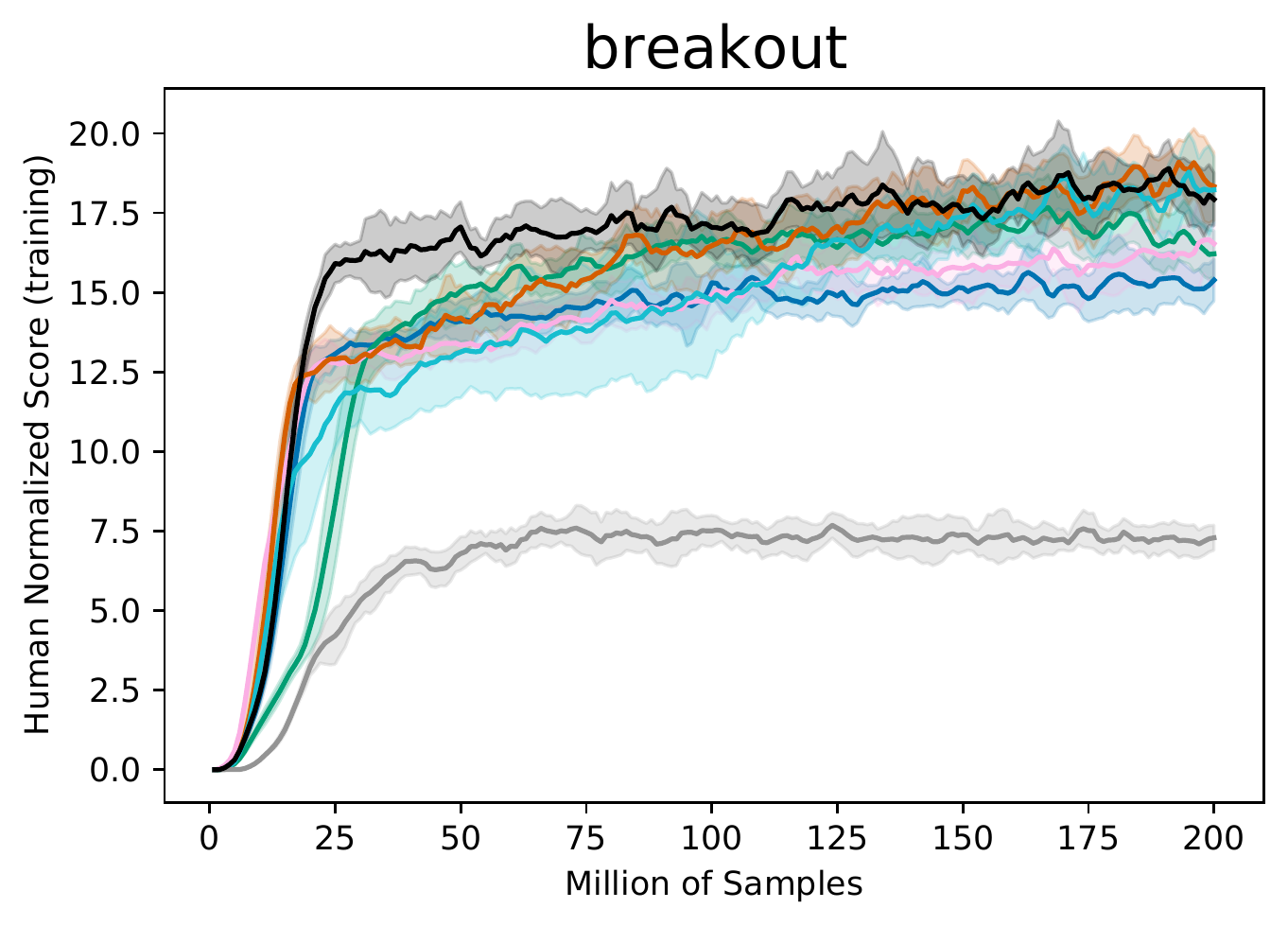}
\includegraphics[align=c,width=.48\columnwidth]{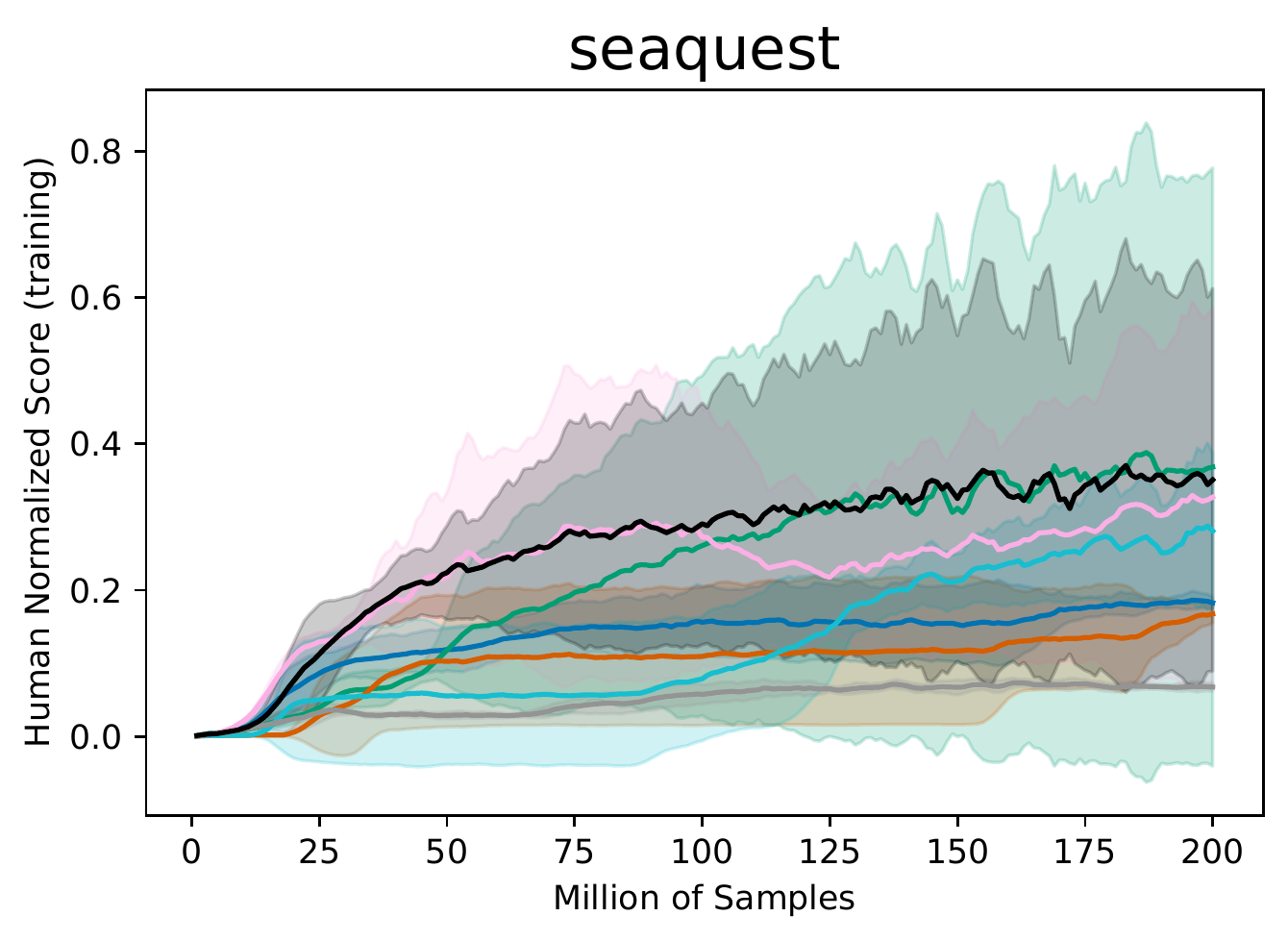}\\
\includegraphics[align=c,width=.48\columnwidth]{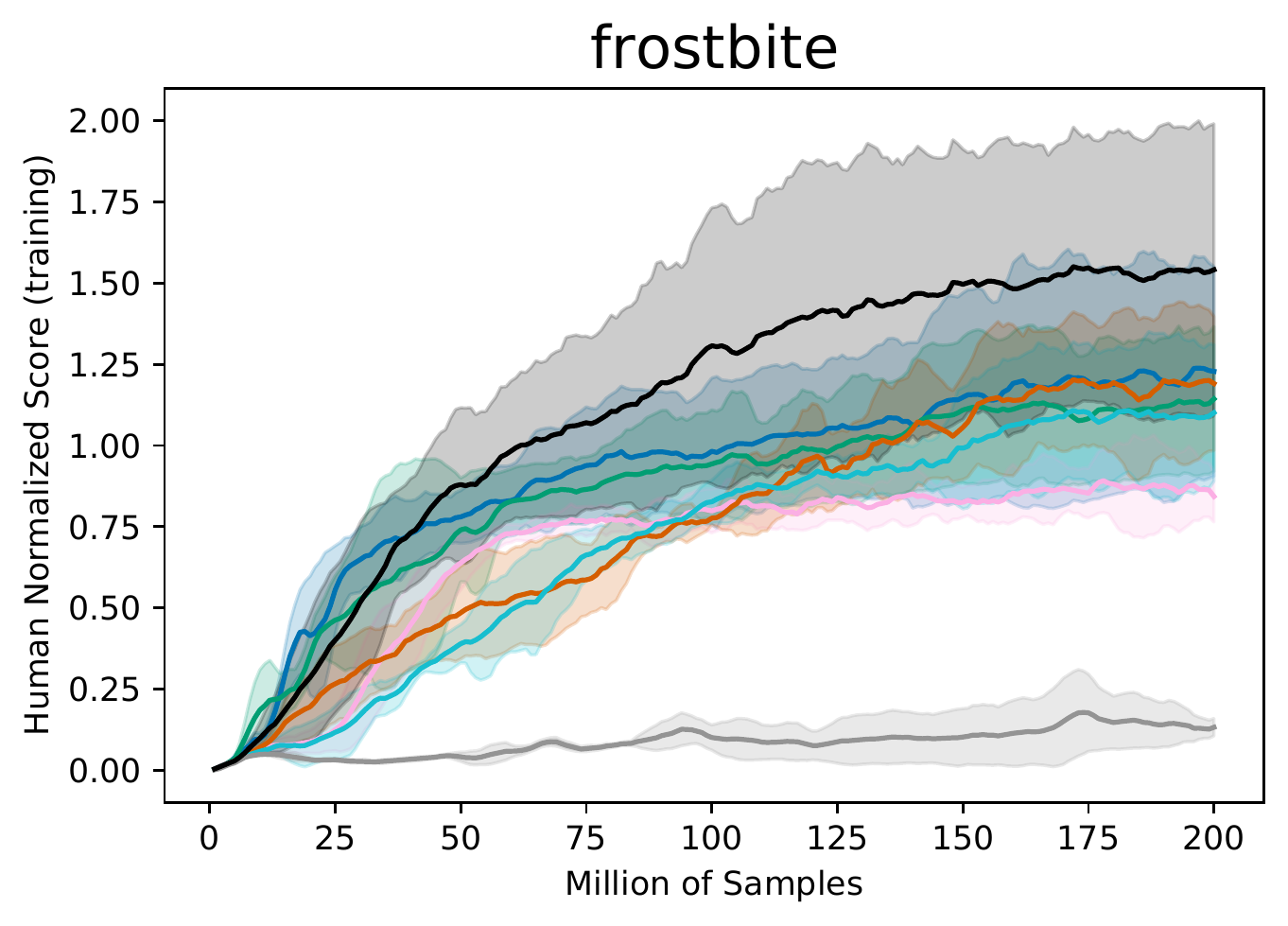}
\includegraphics[align=c,width=.48\columnwidth]{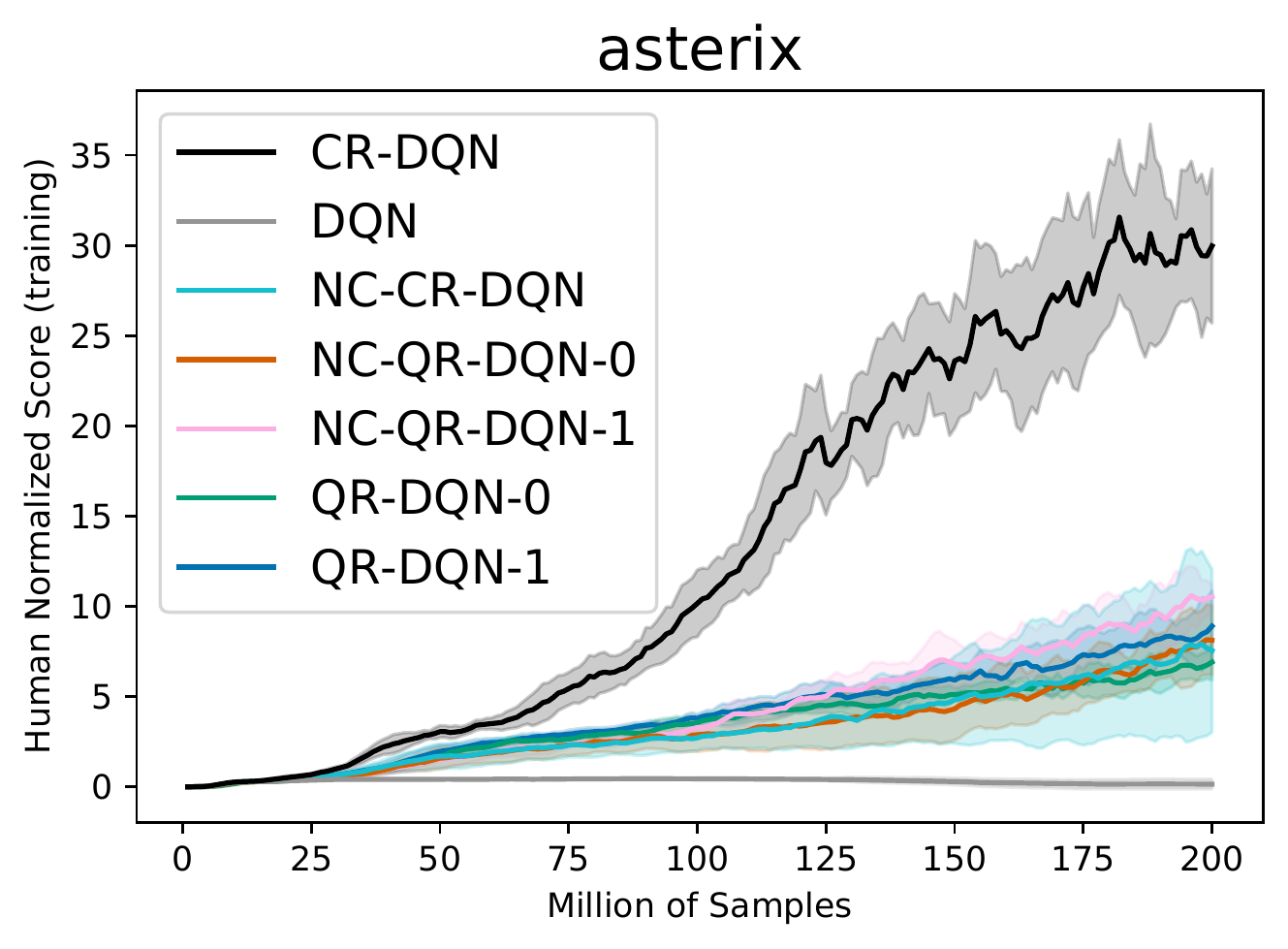}
\caption{{\bf Atari games.}  
CR: Cramér loss. The suffix is the value of $\kappa$. 
NC: non-crossing. %
Curves show mean and std.~dev.~of human-normalized online performance over 3 seeds, smoothed over a sliding window of 5 iter.%
}
\label{fig:atari}
\end{figure} 
We see in Fig.~\ref{subfig:fc-wasserstein} that, due to the biased gradients of the 1-Wasserstein loss, the learned distribution converges to one of the Diracs instead of converging to the mixture.
We also see that the Huberized QR loss yields a shrunken distribution (Fig.~\ref{subfig:fc-qr-loss0.1}), the effect being larger with $\kappa=1$ (Fig.~\ref{subfig:fc-qr-loss1.0}). The standard QR loss (Fig.~\ref{subfig:fc-qr-loss0.0}) and the Cramér one (Fig.~\ref{subfig:fc-cramer}) do not exhibit this effect but we see an oscillation around the actual step locations due to the lack of smoothness.
The Cramér loss exhibits a slightly larger oscillation effect as suggested by the larger 1-Wasserstein distance.
The non-crossing constraints make the QR loss (Fig.~\ref{subfig:nc-qr-loss0.1}) equivalent to the Cramér one (Fig.~\ref{subfig:nc-cramer}) and reduce the oscillation effect but introduce another effect due to the specific architecture. More precisely, the ReLU activation that outputs the scale factor $\alpha$ \citep[Eq.~(19)]{nc-qrdqn} is prone to the dying ReLU problem in this simple setting. This makes the NC architecture converge to one of Diracs in some of the trials. Replacing the ReLU activation by a SoftPlus solves the problem (Fig.~\ref{subfig:nc-softplus-cramer}).
Note that this problem is less likely to happen in more complex scenarios---with more states and actions---as in the Atari games considered next (see \cite{glorot2011deep}).

\subsection{Atari games}

We consider four Atari games exhibiting different learning behaviours. Fig.~\ref{fig:atari} shows the online training performance \citep{machado2018} given by different combinations of networks and losses. 
The NC network \citep{nc-qrdqn} and Algorithm \ref{algo} (denoted CR and used in CR-DQN and NC-CR-DQN) were implemented on top of the \dqnzoo framework \citep{dqnzoo2020github} which also provides pre-computed results for the two reference algorithms QR-DQN (aka QR-DQN-1) \citep{qrdqn} and DQN \citep{dqn}.
Equivalent hyperparameter values were used for all the methods, see Appendix \ref{app:experiments} for details.

Although equivalent in theory, NC-QR-DQN-0 and NC-CR-DQN do not exactly match empirically because of GPU non-determinism and differences in numerical errors. %
See Appendix \ref{app:experiments} for more details.

The permutation invariance of our sort-based algorithm makes the crossing quantile problem vanish, removing the need of non-crossing architectures that are prone to undesired effects as the dying ReLU problem. In these four games, the increased freedom of CR-DQN provides a significant advantage over the other methods with, in particular, a remarkable performance on Asterix. 

To provide comparable results with existing work, we report, in Table \ref{tab:agg-atari}, evaluation results over the full Atari 57 benchmark under
the best agent protocol (see, e.g., \cite{qrdqn}) obtained with the pre-computed results provided in \cite{dqnzoo2020github} for the contenders. We observe that CR-DQN outperforms C51 \citep{c51} and standard QR-DQN \citep{qrdqn}.

\begin{table}

    \centering
    \begin{tabular}{l|l|l}%
             & Seeds & Median \\
    \hline
    DQN      & 5     & 85\%     \\
    C51      & 5     & 183\%    \\
    QR-DQN-1 & 5     & 182\%    \\
    IQN      & 5     & 220\%    \\
    \hline
    CR-DQN   & 3     & 201\%     \\
    \end{tabular}
    \caption{Median of best scores across 57 Atari %
2600 games, measured as percentages of human baseline
\citep{nair2015massively} using reference values from \dqnzoo.}%
    \label{tab:agg-atari}
\end{table}

\section{DISCUSSION}
\label{sec:conclusion}

Our results shed light on QR-based algorithms by showing the equivalence of the Cramér projection with the 1-Wasserstein one, and that learning distributions with the QR loss under non-crossing constraints is essentially equivalent to learning with the Cramér loss.
On the practical viewpoint, we proposed a low complexity algorithm that we tested on synthetic examples and Atari games using an unconstrained architecture and another one with non-crossing constraints.

In the unconstrained setting, symmetries creates a factorial number of optimal solutions (due to the permutations): in a stochastic optimization perspective, this could facilitate (since there are more places to find optimal solutions) and give more freedom to the deep network but it can also make the learning process unstable by jumping from one region to a symmetric one. 
In a constrained setting, Algorithm \ref{algo} computes an output that is equivalent to that of the QR-loss and, thus, it is subject to the same lack of smoothness that has been reported to hurt the performance in comparison to Huberized QR-loss. However, Huberization breaks the equivalence with the Cramér distance and introduces biases whose magnitude depends on the chosen $\kappa$ and the scale of distributions, which can vary from one state to another. Another important point is that the architectures introducing monotonicity constraints can also introduce new effects depending on the design choices.  
As future work, we foresee  investigating alternative approaches to smoothen the Cramér loss.%

\subsubsection*{Acknowledgements}
Thanks to Mourad Boudia, Eoin Thomas and Rodrigo Acuña-Agost for their
insightful comments and to the anonymous reviewers whose suggestions have greatly improved this manuscript.

\medskip
\newpage
{
\small
	\bibliographystyle{plainnat}
	\bibliography{drl}
}

\clearpage
\appendix

\thispagestyle{empty}

\onecolumn \makesupplementtitle

\section{ADDITIONAL PROOFS}
\label{app:proofs}

\auxiliary*
\begin{proof}
A visual intuition of the proof is shown in Figure \ref{fig:midpoint-intuition}.
We decompose the integral as follows
\begin{align}
 \int_{t}^{t'} \lvert F(z) - H_\theta^{\tau,\tau'}(z)\rvert ^p dz 
 &=  \int_{t}^{\theta} (F(z)-\tau)^p dz +  \int_{\theta}^{t'} ( \tau' - F(z))^p dz \label{eq:continuity}\\
 &= \lim_{a\to t} \smallint (F(z)-\tau)^p dz \bigr\vert_a^\theta + 
\lim_{b\to t'} \smallint ( \tau' - F(z))^p dz \bigr\vert_\theta^b \label{eq:fund-thm-calc}
\end{align}
where the limits are taken to cover the particular cases of $t=-\infty$ and $t'=\infty$. The last equation stems from the second fundamental theorem of calculus, which holds since the integrated functions are bounded and the set of points of discontinuity has measure zero (since $F$ is a CDF).
Since we are minimizing with respect to $\theta$ we can drop the constant terms and consider
\begin{align}
\frac{d}{d \theta}  \smallint (F(z)-\tau)^p dz \bigr\vert_\theta - 
\smallint ( \tau' - F(z))^p dz \bigr\vert_\theta
= (F(\theta)-\tau)^p - ( \tau' - F(\theta))^p
.
\end{align}
First note that for $\theta\in[t,t']$, we have $F(\theta)-\tau>0$ and $\tau'-F(\theta)>0$.
Then, equating the derivative to zero yields
\begin{align}
(F(\theta)-\tau)^p - ( \tau' - F(\theta))^p = 0 
\Leftrightarrow F(\theta) -\tau =  \tau' - F(\theta) 
\Leftrightarrow F(\theta) = \frac{\tau + \tau'}{2}.
\end{align}
By replacing $=$ by $<$ (resp., $>$) in the previous equations, we see that the derivative is strictly negative (resp., strictly positive) if $F(\theta)<(\frac{\tau + \tau'}{2})$ (resp., $F(\theta)>(\frac{\tau + \tau'}{2})$), which proves the claim.
If there is a jump in $F$ making $F^{-1}$ undefined at $\frac{\tau + \tau'}{2}$, the set defined in Eq.~\eqref{eq:mid-point-sol} becomes empty.
However, the previous inequalities determining the sign of derivative still hold and the quantity to be minimized is a continuous function of $\theta$ (see Eq.~\eqref{eq:continuity}). Therefore, if we redefine $F^{-1}$ to be the inverse CDF, it makes $F^{-1}((\tau+\tau')/2)$ always a valid minimizer.
NB: if the standard inverse $F^{-1}$ is undefined at $\tau$ or $\tau'$, the whole derivation still holds if $F^{-1}$ is redefined as the inverse CDF. 
\end{proof}

\section{CORRECTNESS OF ALGORITHM \ref{algo}}
\label{app:algo}

\begin{figure}[H]
\centering
~\includegraphics[align=c,width=.4\columnwidth]{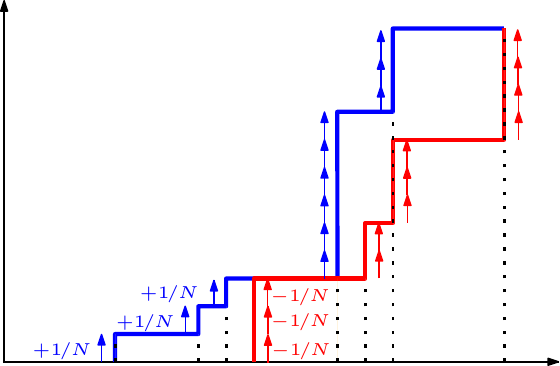}
\caption{{\bf Cramér loss algorithm.} Illustration of $\Delta_\tau$ computation by accumulating increments/decrements. 
 }
\label{fig:algo}
\end{figure} 

\begin{figure}[t!]
\centering
\begin{subfigure}[t]{.49\textwidth}
    \centering
    \includegraphics[width=.7\linewidth]{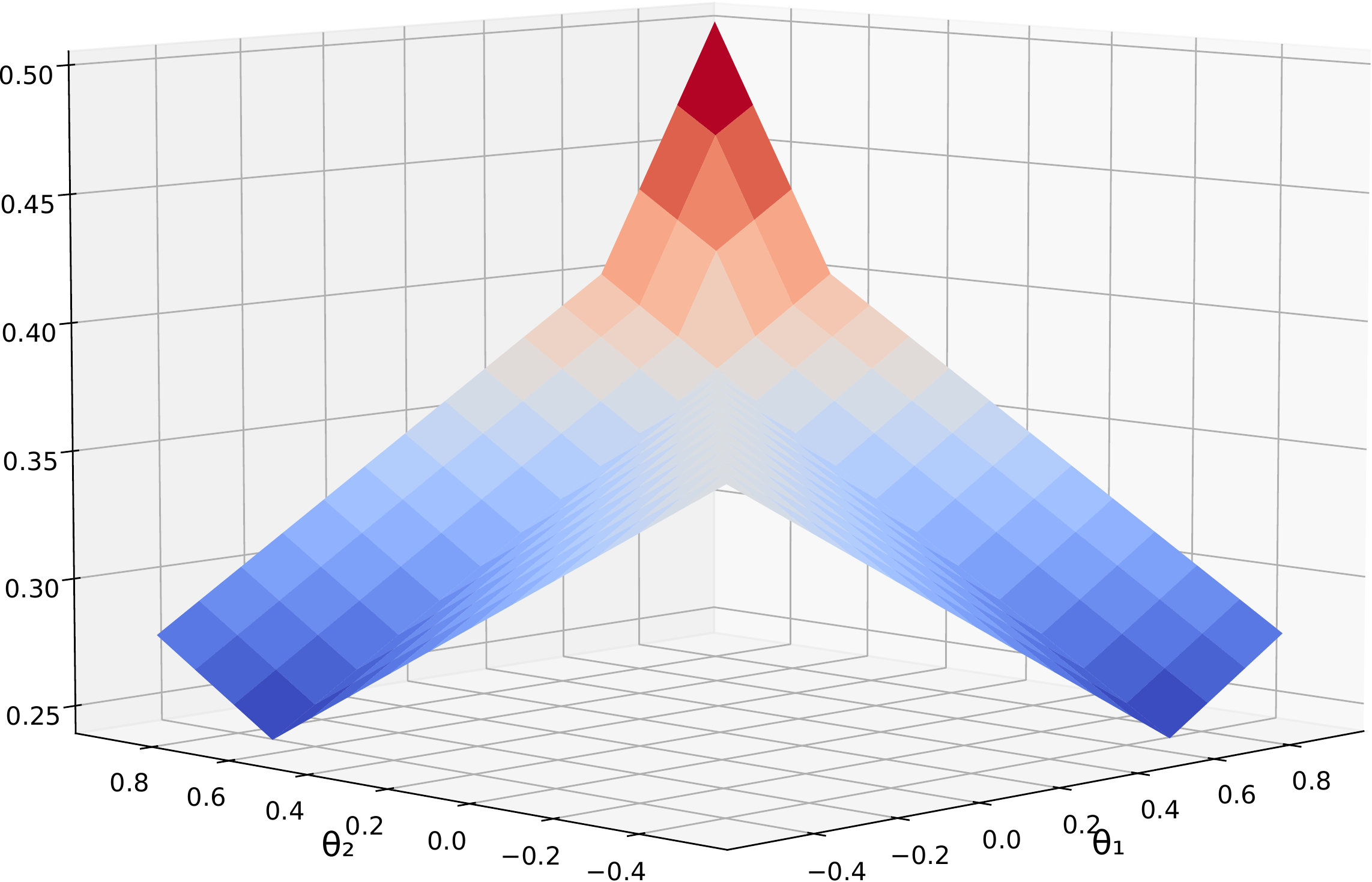}
    \caption{Cramér loss.}
\end{subfigure}
~
\begin{subfigure}[t]{.49\textwidth}
    \centering
    \includegraphics[width=.7\linewidth]{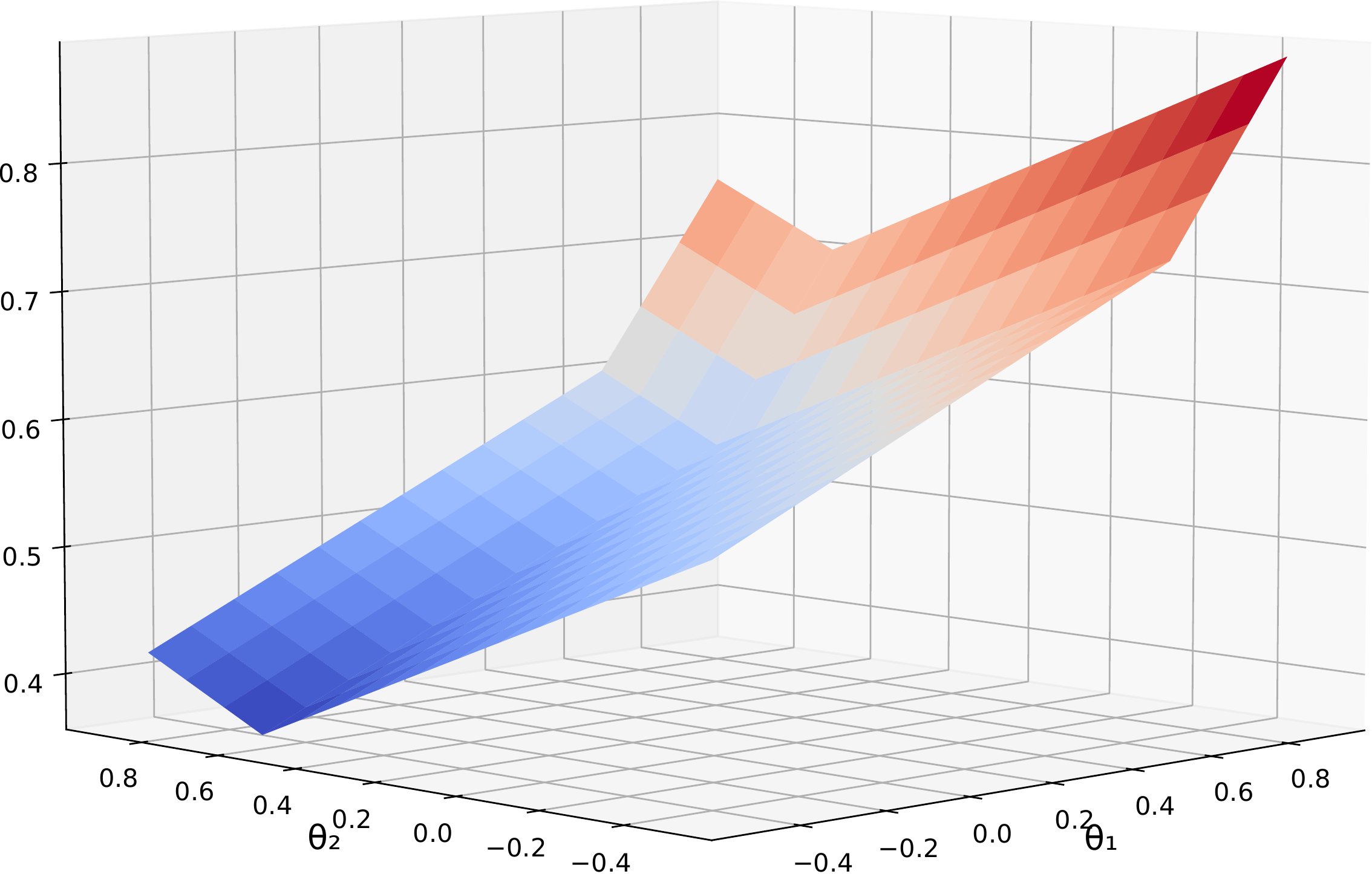}
    \caption{Quantile regression loss.}    
\end{subfigure}
\caption{{\bf Symmetry in the Cramér loss landscape (a) in comparison to the QR loss (b).}  The loss landscape correspond to estimating the return distribution of a state $s_0$ with transitions to states $s_1$ and $s_2$ with probability $1/3$ and $2/3$, respectively,  whose return distributions are Diracs located at $-0.5$ and $0.6$ respectively, with $N=3$. The plots are for a fixed $\theta_0=-0.5$. Notice that when $\theta_0\leq\theta_1\leq\theta_2$,  the two losses have collinear gradients as shown in Corollary \ref{cor:col-grad}.}
\label{fig:symmetries}
\end{figure}

\begin{restatable}[]{prop}{correctness}
Given two distributions $F(z)=\frac{1}{N}\sum_{i=1}^N  \ind_{z\geq \theta_i}$, and $\TF(z)=\frac{1}{N}\sum_{i=1}^N  \ind_{z\geq \Ttheta_i}$, Algorithm \ref{algo} computes 
\begin{equation}
\label{eq:algo-formula}
\int_{-\infty}^\infty (F(z) - \TF(z))^2 dz = \sum_{i=1}^{2N-1}  \left( \theta'_{i+1} - \theta'_{i}  \right)\left( \sum_{j \text{ s.t. } \theta_j\leq\theta'_i} \frac{1}{N} - \sum_{j \text{ s.t. } \Ttheta_j\leq\theta'_i} \frac{1}{N} \right)^2    
.
\end{equation}
\end{restatable}
\begin{proof}
Consider the sorted sequence of merged quantiles 
\begin{equation}
\vtheta'\equiv \theta'_1,\dots,\theta'_{2N} \equiv \mathrm{sort}\left(\{\theta_i\}_{i=1..N}\bigcup\{\Ttheta_i\}_{i=1..N}\right)
.
\end{equation}
We have that $F(z) - \TF(z)\equiv\Delta_i$ is constant in $[\theta'_i,\theta'_{i+1}), \forall i\in 1..2N-1$ and is zero elsewhere. 
Therefore,
\begin{equation}
\int_{-\infty}^\infty (F(z) - \TF(z))^2 dz =
\sum_{i=1}^{2N-1} \int_{\theta'_i}^{\theta'_{i+1}} (F(z) - \TF(z))^2 dz =
\sum_{i=1}^{2N-1} \Delta_i^2(\theta'_{i+1} - \theta'_i)  
\end{equation}
If $\theta'_i\leq z < \theta'_{i+1}$, then 
\begin{align}
F(z)&=\frac{1}{N}\sum_{j=1}^N  \ind_{z\geq \theta_j} = \frac{1}{N} \sum_{j \text{ s.t. } \theta_j\leq\theta'_i} 1 \\
\TF(z)&=\frac{1}{N}\sum_{j=1}^N  \ind_{z\geq \Ttheta_j} = \frac{1}{N} \sum_{j \text{ s.t. } \Ttheta_j\leq\theta'_i} 1
\end{align}
and thus
\begin{equation}
\Delta_i =  \sum_{j \text{ s.t. } \theta_j\leq\theta'_i} \frac{1}{N} - \sum_{j \text{ s.t. } \Ttheta_j\leq\theta'_i} \frac{1}{N} 
,
\end{equation}
which proves \eqref{eq:algo-formula}.

The algorithm computes the differences $(\theta'_{i+1} - \theta'_i)$ and stores them in $\Delta_z$.
After the steps
\begin{align}
    \Delta_\tau&\leftarrow \mathrm{concat}\left(-\frac{1}{N}\mathbf{1}_N,\frac{1}{N}\mathbf{1}_N\right)\\
    \Delta_\tau&\leftarrow \Delta_\tau[i_1,\dots,i_{2N}]
    ,
\end{align}
in words, the $i$-th element of the vector $\Delta_\tau$ is $-\frac{1}{N}$ if $\theta'_i$ comes from $\vTtheta$ or $\frac{1}{N}$ otherwise, i.e.
\begin{equation}
    \Delta_\tau[i]= \frac{1}{N} (-1)^{\ind_{\exists j \theta'_i\equiv\Ttheta_j}}
\end{equation}
where $\equiv$ denotes symbol equality. See Fig.~\ref{fig:algo} for an illustration. After the final step
\begin{equation}
\Delta_\tau\leftarrow \mathrm{cumsum}\left(\Delta_\tau\right)[:\text{-}1],
\end{equation}
the $i$-th element of the vector $\Delta_\tau$ can be expressed as
\begin{align}
\label{eq:cumsum-i}
    \Delta_\tau[i]&=\frac{1}{N}\sum_{k=1}^i (-1)^{\ind_{\exists j \theta'_k\equiv\Ttheta_j}}
.
\end{align}
If $\theta'_i\neq\theta'_{i+1}$, then $\Delta_\tau[i]=\Delta_i$. Otherwise, $\Delta_\tau[i]\neq\Delta_i$, but, since $\theta'_{i+1}-\theta'_i=0$, the corresponding term in \eqref{eq:algo-formula} is zero too.
Therefore, the algorithm produces the claimed output.
\end{proof}

\section{EXPERIMENTAL DETAILS}
\label{app:experiments}

\subsection{The networks}

We describe here the two types of architecture used in the experiments. See Figure \ref{fig:architectures} for an illustration.
\begin{figure}
    \centering
    \begin{subfigure}{\textwidth}
        \centering
        \includegraphics{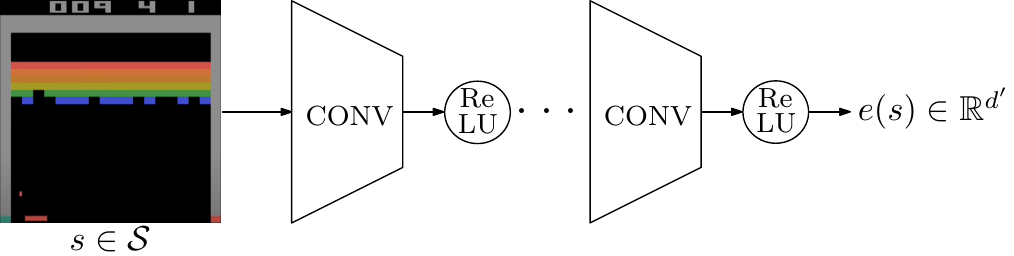}
        \caption{DQN backbone: feature extraction by a series of convolutional layers with ReLU activations.}
        \label{subfig:torso}
    \end{subfigure}
    \par\bigskip %
    \begin{subfigure}{\textwidth}
        \centering
        \includegraphics{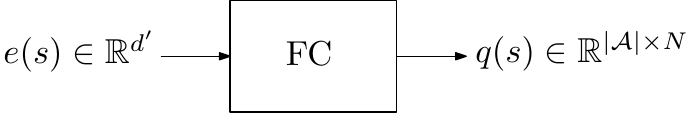}
        \caption{QR-DQN head: a fully-connected network.} 
        \label{subfig:dqn-head}
    \end{subfigure}
    \par\bigskip %
    \begin{subfigure}{\textwidth}
        \centering
        \includegraphics{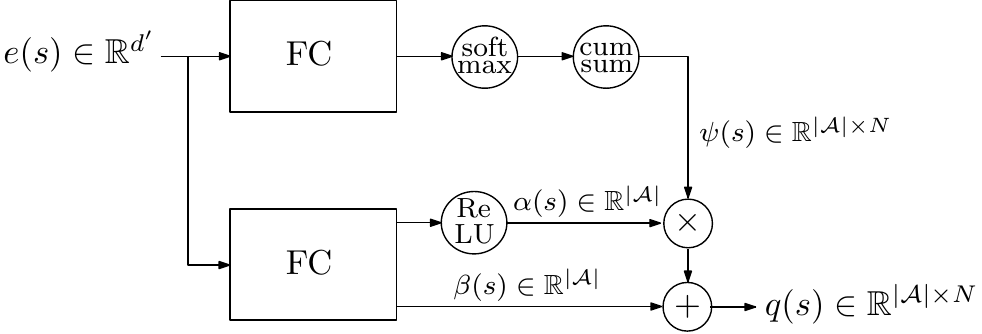}
        \caption{Non-Crossing (NC) head: combination of NCQL (upper part) and SF (lower part) networks.}
        \label{subfig:nc-head}
    \end{subfigure}
    \caption{{\bf Architectures used in the experiments.}}
    \label{fig:architectures}
\end{figure}
QR-DQN \citep{qrdqn} uses a series of convolutional layers each one followed by a ReLU activation in order to extract features from the input frames to obtain an embedded state $e(s)\in\R^{d'}$ (Fig.~\ref{subfig:torso}). They are followed by a fully connected network with $\lambda$ layers of $\eta$ nodes each and an output layer of size $\abs{\mathcal{A}}\times N$ (Fig.~\ref{subfig:dqn-head}). 

Following \cite{nc-qrdqn}, the NC network used in the experiments replaces the fully connected network of QR-DQN by a \emph{Non-Crossing Quantile Logit} (NCQL) network and a \emph{Scale Factor} (SF) network (Fig.~\ref{subfig:nc-head}). The NCQL network maps the embedded state $e(s)$ to $\abs{\mathcal{A}}\times N$-dimensional logits by using a fully connected network of $\lambda$ layers with $\eta$ nodes each, which is followed by a softmax transformation. Then a cumulated sum operator produces a non-decreasing sequence of normalized quantile values $\psi(s)[a,1],\dots,\psi(s)[a,N]$ for each action $a$.
The SF network produces an output in $\abs{\mathcal{A}}\times2$ representing the scale $\alpha(s)[a]$ and the location $\beta(s)[a]$ of the CDF, by mapping the embedded state $e(s)$ through a fully-connected network of $\lambda$ layers and $\eta$ nodes. A ReLU function is applied to the output corresponding to the scale $\alpha(s)[a]$ to ensure its non-negativity.
The final quantile estimates are obtained by combining the outputs of the two networks as follows
\begin{equation}
q(s)[a,i]:=\alpha(s)[a] \times \psi(s)[a,i]+\beta(s)[a]\;; i=1, \ldots, N , a=1, \ldots,\abs{\mathcal{A}} .
\end{equation}

Since in the synthetic experiment there is only one state, the feature extraction layers are removed and therefore QR-DQN turns into a standard fully-connected (FC) architecture. The NC architecture in this case boils down to the combined NCQL and SF networks.

\subsection{Synthetic experiment}
This experiment considers an MDP with only one possible action in one state $s$ that can transition to two possible states $s_1$ and $s_2$ with probabilities $2/3$ and $1/3$, respectively, each with a different return distribution---a Dirac located at -1 and 1 respectively. The goal is to learn the return distribution at $s$.

Since we aim at learning the return distribution of only one state, the two networks (FC and NC) take a constant scalar input 1. The FC and NC networks have $\lambda=2$ hidden layers of $\eta=45$ and $\eta=32$ nodes, respectively, and an output of $N=12$ quantiles allowing to represent the mixture exactly.

We use the Adam optimizer \citep{adam} with a learning rate of \num{1e-3} and a batch size of 32.

\subsection{Atari games}

We implemented our algorithm on top of the \dqnzoo \citep{dqnzoo2020github} framework, which integrates reference implementations of RL algorithms with the gym/atari-py RL environment \citep{brockman2016openai}. \dqnzoo provides pre-computed simulation results for each of these algorithms, each of them being run on 5 seeds and on the full set of 57 Atari 2600 games. 

In order to implement the NC architecture, we replaced the fully connected network in the \dqnzoo implementation of QR-DQN by the combination of the NCQL and SF networks.

\paragraph{Hyperparameters} For model training, we set our hyperparameters with the values used in \cite{qrdqn} for the epsilon decay and experience replay settings. Notice that ADAM's invariance (cf. Remark \ref{rmk:adam-invariance}) is broken with the parameter $\epsilon$ used in the update step to avoid divisions by zero \citep{adam}: $\theta_{t} \leftarrow \theta_{t-1}-\alpha \cdot \widehat{m}_{t} /\left(\sqrt{\widehat{v}_{t}}+\epsilon\right)$, where $\widehat{m}_{t}$ and $\widehat{v}_{t}$ are the first and second moment estimates at timestep $t$, which are scaled by a factor of $c$ and $c^2$ respectively when the gradient is scaled by $c$. Since the gradient of the Cramér loss is $c=2/N$ times the one of the QR loss (cf.~Corollary \ref{cor:col-grad}), we use the adjusted $\epsilon'\equiv (\nicefrac{2}{N})\epsilon$ to have equivalent update steps. Each experiment consists in 200 iterations. Each iteration is made of a learning phase (1 million frames), followed by an evaluation phase, on 500 thousands frames. We thus use the same experiment procedure, and the same epsilon hyperparameter than the one used for the experiments provided with \dqnzoo; also, our neural network architecture uses the same three convolutional layers as the other algorithms implemented within \dqnzoo. The experiment settings being the same, our experiment performance can therefore be compared to the experiment data provided with \dqnzoo for the other algorithms. 
Finally, the neural networks are defined by $\lambda=1$, $\eta=512$ and $N=201$.  Table \ref{tab:hyperparams} summarizes the hyperparameters and their values.

\paragraph{Online training performance}
Performance during training protocol: this protocol, described in \cite{machado2018}, puts the emphasis on the learning quality. It consists in using normalized training scores to evaluate the algorithms.
Human-normalization of score is given by \cite{Hasselt2016}: $\mathrm{normalized\_score} = \frac{\mathrm{agent\_score-random}}{\mathrm{human-random}}$ where $\mathrm{random}$ and $\mathrm{human}$ are baseline scores, given for each game.

\begin{table}
\caption{{\bf Hyperparameters for *-\{C|Q\}R-DQN methods.}}
\small
\centering
\label{tab:hyperparams}
\begin{tabular}{lll}
\toprule
Hyperparameter                               & Value     & Comment           \\
\midrule
replay\_capacity                             & 1e6       &                   \\
min\_replay\_capacity\_fraction              & 0.05      & Min replay set size for learning    \\
batch\_size                                  & 32        &                   \\
max\_frames\_per\_episode                    & 108000    & = 30 min          \\
num\_action\_repeats                         & 4         & In frames         \\
num\_stacked\_frames                         & 4         &                   \\
exploration\_epsilon\_begin\_value           & 1         &                   \\
exploration\_epsilon\_end\_value             & 0.01      &                   \\
exploration\_epsilon\_decay\_frame\_fraction & 0.02      &                   \\
eval\_exploration\_epsilon                   & 0.001     &                   \\
target\_network\_update\_period              & 4e4       &                   \\
learning\_rate                               & 5e-5      &                   \\
optimizer\_epsilon (for *-CR-DQN and NC-QR-DQN-0)                          & 0.01 / 32 * 2/$N$ & ADAM's parameter \\
optimizer\_epsilon (otherwise)          & 0.01 / 32  & ADAM's parameter   \\

additional\_discount                         & 0.99      & Discount_rate multiplier \\
max\_abs\_reward                             & 1         &                   \\
max\_global\_grad\_norm                      & 10        & Gradient clipping  \\
num\_iterations                              & 200       &                   \\
num\_train\_frames                           & 1e6       & Per iteration    \\
num\_eval\_frames                            & 5e5       & Per iteration      \\
learn\_period                                & 16        & One learning step each 16 frames \\
num\_quantiles                               & 201       &  $N$               \\
\midrule
Convolutional layer 1 & 32, (8, 8), (4, 4) & num_features, kernel_shape, stride\\
Convolutional layer 2 & 64, (4, 4), (2, 2) & \\
Convolutional layer 3 & 64, (3, 3), (1, 1) & \\
\midrule
n\_layers                                    & 1         &  Number of hidden layers $\lambda$         \\
n\_nodes                                     & 512       &   Number of nodes $\eta$ per hidden layer            \\
\bottomrule
\end{tabular}
\end{table}

\paragraph{Detailed results}
Figure \ref{fig:all-games} shows the online training performance of CR-DQN in comparison to the pure distributional contenders C51, QR-DQN (aka QR-DQN-1) and IQN, on the full Atari-57 benchmark. For C51, QR-DQN and IQN, 5 seeds were used (provided by \dqnzoo \cite{dqnzoo2020github}). For CR-DQN, 3 seeds were used.  

\begin{figure}
\centering
\includegraphics[height=.94\textheight]{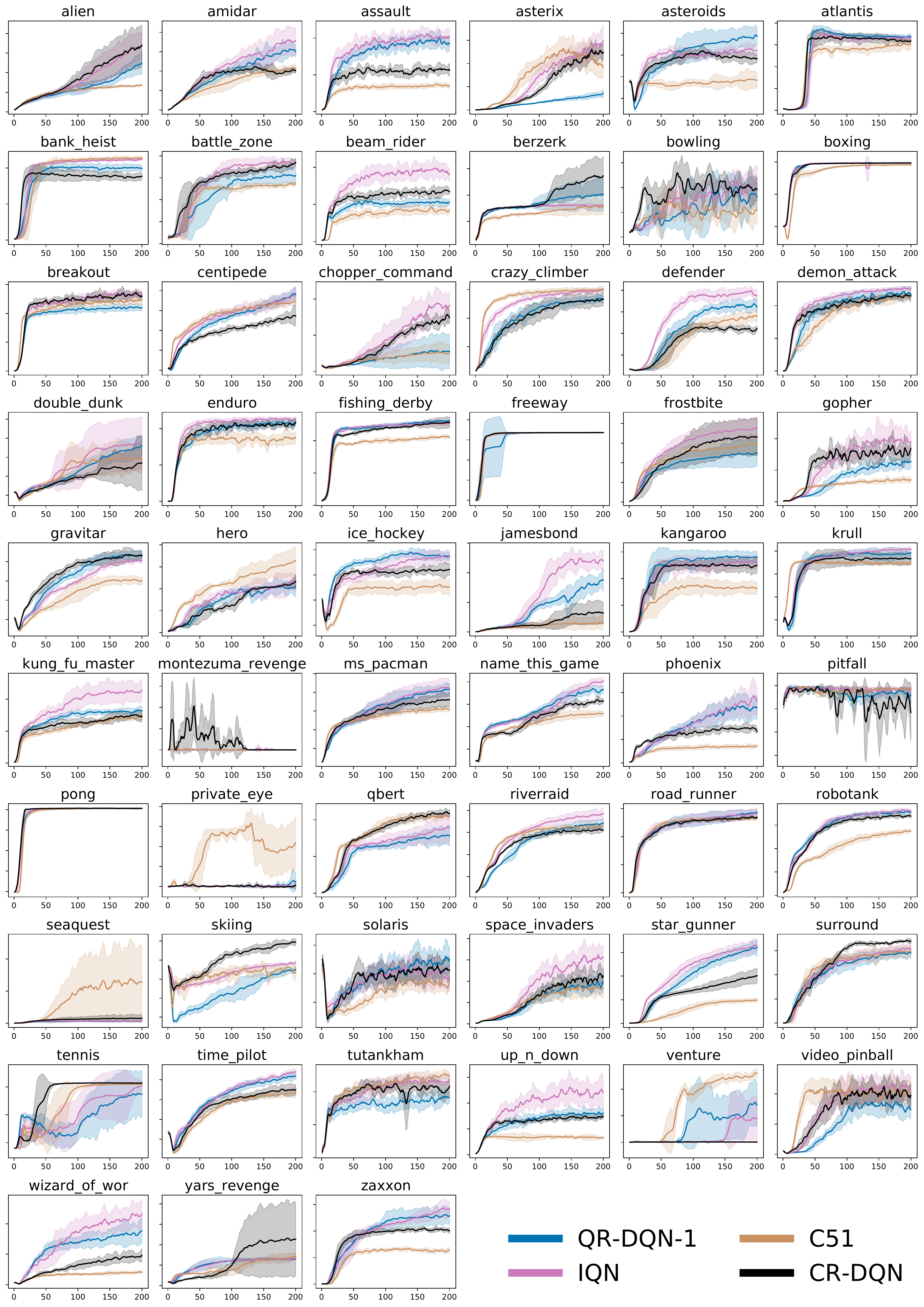}
	\caption{{\bf Training performance on the Atari-57 benchmark.} Curves are averages over a number of seeds, smoothed over a sliding window of 5 iterations, and error bands give standard deviations.  }
	\label{fig:all-games} 
\end{figure} 

\paragraph{On the empirical matching of NC-QR-DQN-0 and NC-CR-DQN}
In order to make NC-QR-DQN-0 and NC-CR-DQN as practically equivalent as possible for the experiments of Figure \ref{fig:atari}, the gradient of NC-QR-DQN-0 was scaled by a factor of $2/N$ to make the effect of gradient clipping by max\_global\_grad\_norm equivalent and the same optimizer\_epsilon was used (see Table \ref{tab:hyperparams}). 
Despite this, numerical errors and GPU non-determinism still produce different results.

\end{document}